\documentclass[Afour,sageh,times]{sagej}
\usepackage{moreverb,url}
\usepackage{graphics} 
\usepackage{graphicx} 
\usepackage{epsfig} 
\usepackage{mathptmx} 
\usepackage{times} 
\usepackage{amsmath} 
\usepackage{amssymb}  
\usepackage{algorithm}
\usepackage{algorithmic}
\usepackage{subcaption}
\usepackage{bm}
\usepackage{natbib}
\usepackage{amsthm}

\usepackage{todonotes}

\newcommand{\boxalg}{{\sc box}}
\newcommand{\randalg}{{\sc rand}}
\newcommand{\dooalg}{{\sc doo}}
\newcommand{\staticalg}{{\sc static}}
\newcommand{\rawalg}{{\sc rawplanner}}

\usepackage{amsopn,amssymb,thmtools,thm-restate}
\newtheorem{thm}{Theorem}
\newtheorem{lem}[thm]{Lemma}

\newtheorem{cor}[thm]{Corollary}

\DeclareMathOperator{\argmax}{arg\,max}
\DeclareMathOperator{\argmin}{arg\,min}

\newcommand{\field}[1]{\mathbb{#1}}
\newcommand{\R}{\field{R}} 
\newcommand{\T}{^{\textrm T}} 

\providecommand{\dif}{\mathop{}\!\mathrm d}

\newcommand{\mat}[1]{\boldsymbol{#1}} 
\newcommand{\mI}{\mat{I}}

\newcommand{\hS}{\hat{\Sigma}}

\newcommand{\hmu}{\hat{\mu}}

\def\volumeyear{2018}

\begin{document}

\runninghead{Kim et al.}

\title{Learning to guide task and motion planning using score-space representation}

\author{Beomjoon Kim, Zi Wang, Leslie Pack Kaelbling and
 Tom\'as Lozano-P\'erez\affilnum{1}}

\affiliation{\affilnum{1} Massachusetts Institute of Technology, 
Computer Science and Artificial Intelligence Laboratory, USA}

\corrauth{Beomjoon Kim, Massachusetts Institute of Technology,
Computer Science and Artificial Intelligence Laboratory,
Cambridge, MA, 02139, USA}
\email{beomjoon@mit.edu}

\begin{abstract}

In this paper, we propose a learning algorithm that speeds up the
search in task and motion planning problems.  
Our algorithm proposes solutions to three different
challenges that arise in learning to improve planning efficiency: 
what to predict,  how to represent a planning problem instance, 
and how to transfer knowledge from one problem instance to another. 
We propose a method that predicts constraints on the search space
based on a generic representation of a planning problem instance,
called score-space, where we represent a problem instance in terms of the
performance of a set of solutions attempted so far. Using this representation,
we transfer knowledge, in the form of constraints, from previous problems based on
the similarity in score space.
We design a sequential algorithm that efficiently
predicts these constraints, and evaluate it in three
different challenging task and motion planning problems. Results indicate that 
our approach performs orders of magnitudes faster than an unguided planner. 
\end{abstract}

\maketitle
\newcommand{\probspace}{\Omega}
\newcommand{\prob}{\omega}
\newcommand{\minsolconspace}{\Theta_{\min}}
\newcommand{\solcon}{\theta}
\newcommand{\optsolcon}{\theta_*}
\newcommand{\testprob}{\prob_{n+1}}
\newcommand{\expmat}{\mathbf{D}}
\newcommand{\minliblist}{\mathcal{L}}
\newcommand{\nextsolconidx}{\theta_{next}}

\newcommand{\solspace}{\mathcal{X}}
\newcommand{\apprxsolconspace}{\Theta}
\newcommand{\solconspace}{\Theta}

\newcommand{\origset}{\Theta}
\newcommand{\regsetx}{\bar{x}}
\newcommand{\minsetx}{x}
\newcommand{\minset}{\Theta_{min}}
\newcommand{\bestminset}{\Theta_{\min}^*}
\newcommand{\minsetlist}{\minliblist^*}
\newcommand{\gpucb}{{\sc gpucb}}

\begin{abstract}

In this paper, we propose a learning algorithm that speeds up the
search in task and motion planning problems.  
Our algorithm proposes solutions to three different
challenges that arise in learning to improve planning efficiency: 
what to predict,  how to represent a planning problem instance, 
and how to transfer knowledge from one problem instance to another. 
We propose a method that predicts constraints on the search space
based on a generic representation of a planning problem instance,
called score-space, where we represent a problem instance in terms of the
performance of a set of solutions attempted so far. Using this representation,
we transfer knowledge, in the form of constraints, from previous problems based on
the similarity in score space.
We design a sequential algorithm that efficiently
predicts these constraints, and evaluate it in three
different challenging task and motion planning problems. Results indicate that 
our approach performs orders of magnitudes faster than an unguided planner. 
\end{abstract}

\section{1. Introduction}

Task and motion planning (TAMP) problems are sequential
decision making problems in which a robot is required to
search for a sequence of decisions that account for both
discrete and continuous aspects of the world to achieve 
a high-level goal. These decisions are intricately related to
one another, and they must abide by collision-constraints and
transition dynamics.

 A variety of planners have been developed
for TAMP problems~\citep{GravotISRR05,TLPKBel13,TLPIROS14,SrivastavaICRA14,
ToussaintIJCAI15,DantamRSS17}.
However, their worst-case computation time generally scales
exponentially with problem size, and each new problem instance must be
solved from scratch, making them inefficient for real-world tasks.
 In contrast, humans are able to short-cut their
planning process by learning to adapt previous planning experience to
reduce the search space intelligently for new problem instances.
This observation motivates the design of an algorithm that learns from
experience to make predictions that guide the search of a
planner.  We face three important questions in designing the learning
algorithm: (1) what to predict, (2) how to represent a problem
instance, and (3) how to transfer knowledge from past experience to
the current problem instance.

\begin{figure}
\centering
\includegraphics[height=4.5cm,width=4.2cm]{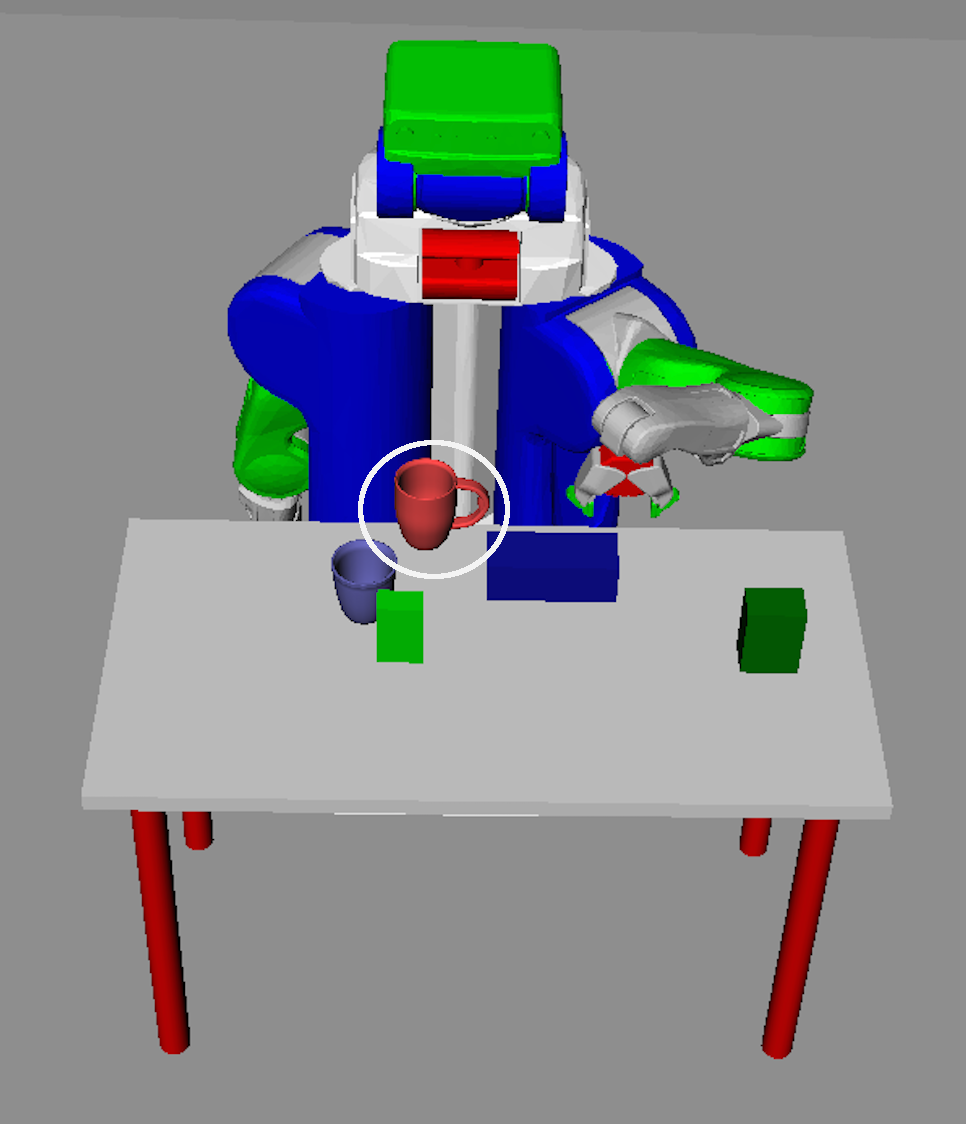}
\includegraphics[height=4.5cm,width=4.2cm]{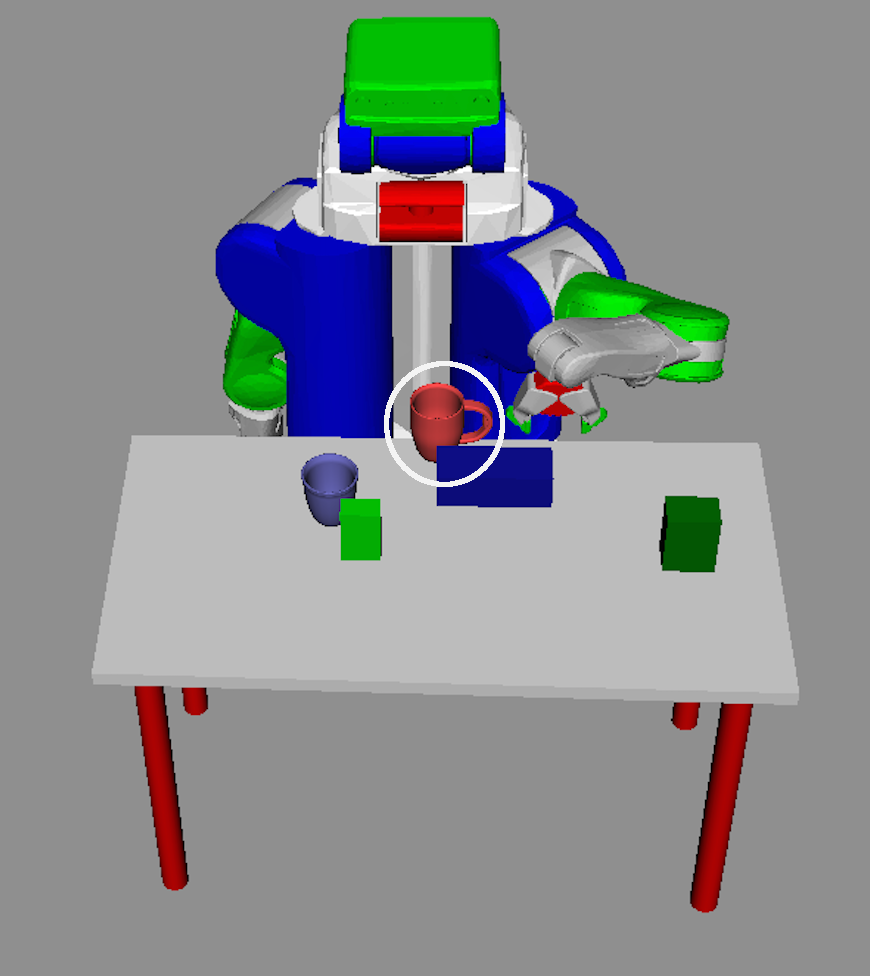}\\
\caption{The task is to pick up the blue object by planning a collision-free path from an
 initial configuration to a pre-grasp configuration. When the obstacle (pink cup) moves 
even by 0.02m, this requires a qualitatively different collision-free path.}
\label{fig:motiv2}
\end{figure}

The first challenge is what to predict.  Previous approaches to using
learning to speed up planning have tried predicting a complete solution,
or a subgoal that fully specifies the robot configuration and world state,
including object poses. However, because of the intricate 
relationship between object poses and the robot's free-space, a small
change in the environment may completely alter the space of feasible
solutions.  This lack of regularity in the relationship between a problem
instance and its solution makes it difficult to predict a complete
solution or a subgoal based on experience. This difficulty is 
illustrated in Figure~\ref{fig:motiv2} in a pick domain.

\begin{figure}
\centering
\includegraphics[scale=0.2]{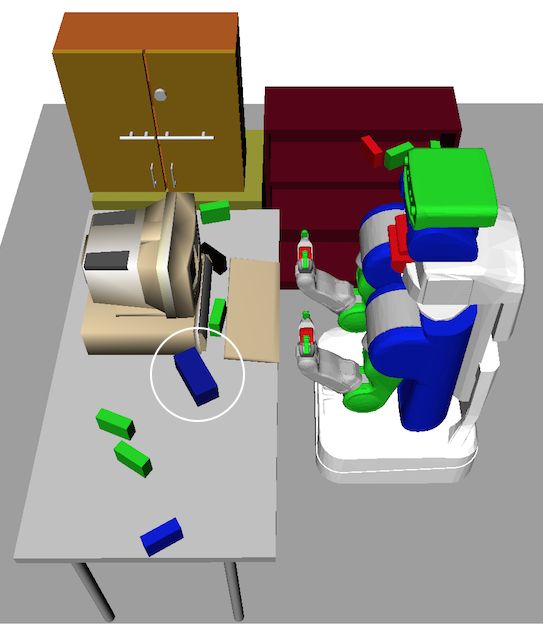}
\includegraphics[scale=0.2]{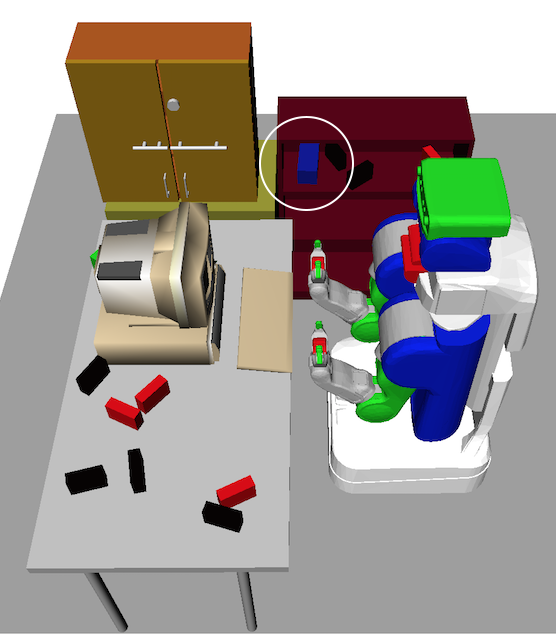}
\caption{Two instances of the grasp selection domain. The arrangement and number of
  obstacles vary randomly across different planning problem instances. The objective
is to find an arm trajectory to a pre-grasp pose for the blue box, marked with a circle, whose pose
is randomly determined in each problem instance}
\label{fig:grasp_domain}
\end{figure}

Building on this observation, we instead learn to predict
constraints on the solution by finding a subset of decision
variables that can be predicted reliably, while
 leaving the rest of the decision variables to be
filled in by the planner. This decomposition is based on the intuition
that constraints can generalize more effectively across problem
instances than a complete solution or a subgoal. For instance,
consider a robot trying to pick an object from a shelf, as shown in
Figure \ref{fig:grasp_domain} (right). A constraint that forces
the robot to approach the object either from the side or the top, depending 
on whether the object is on the table or the shelf, can be used more reliably 
across different arrangements of obstacles and object poses than a detailed 
path plan or a specific pre-grasp configuration.

We will refer to the subset of decision variables that we predict as
\emph{solution constraints}, or constraints for short.  
Solution constraints, when intelligently
chosen, will effectively reduce the search space while preserving the
robustness of the planner against changes in problem instances. 

These points are illustrated in Figure~\ref{fig:constraint_illustration}.
Notice how the reduced search space that satisfies the given constraint
is much smaller than the original search space. 
The planner now only has to find a solution within this space,
requiring much less computation than the unconstrained planner.
Also notice that unlike a complete plan, a constraint is more
general in the sense that a single constraint can be applied to
a set of problems instead of a single problem.

\begin{figure}
\centering
\includegraphics[scale=0.3]{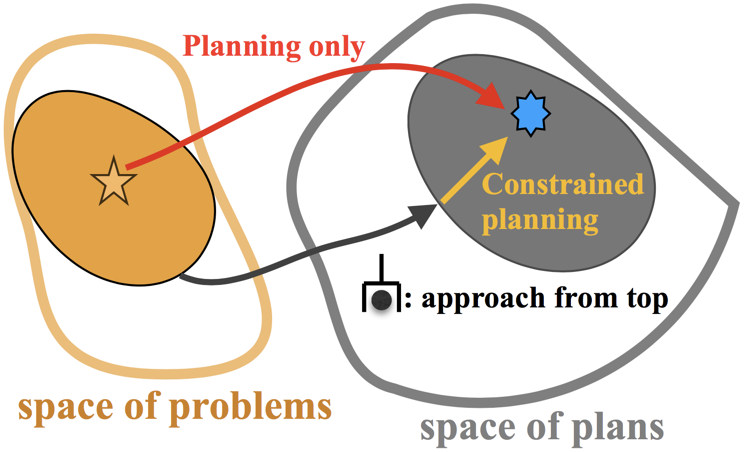}
\caption{A comparison of search done by a planner
without a constraint, and with constraint that forces
the planner to find a solution that approaches the target object from
the top. The black arrow indicates the work done by a machine learning algorithm,
which predicts a constraint, and the yellow arrow indicates 
the work done by a constrained planner.
The gray blob in the space of plans show the reduced search space
that satisfies the constraint. The yellow blob in the space of
problems indicate the set of problem instances for which
there exists a plan that satisfies the top-grasp constraint. 
The yellow star indicates the given problem instance.}
\label{fig:constraint_illustration}
\end{figure}

The second challenge is representing problem instances. In most
of the TAMP problems that we are contemplating, a manual design of 
a generic feature representation for predicting constraints is difficult: there are 
varying numbers and shapes of objects for each problem instance, 
and some object relations matter for some problem instances but
not for others. For instance, in Figure \ref{fig:grasp_domain} (left),
the obstacles on the desk influence how the robot
should pick the blue object and the obstacles on the shelf are
irrelevant. On the other hand, in Figure  \ref{fig:grasp_domain} (right),
the obstacles on the desk are irrelevant.

In light of this, we propose a new type of representation for problem
instances, called \emph{score-space}. In a score-space representation,
 we represent a problem instance
in terms of a vector of the scores for a set of plans on that
instance, where each of the plans is computed based on one of a fixed
set of promising solution constraints. The intuition is that what matters in a 
planning problem is how an environment responds to a potential solution, and a 
good way to predict what solution constraints will work well is to consider
how effective other solution constraints have been. 
Figure~\ref{fig:score_space_illustration}
shows examples of score-space representations of different 
TAMP problem instances. 
\begin{figure}
\centering
\includegraphics[scale=0.3]{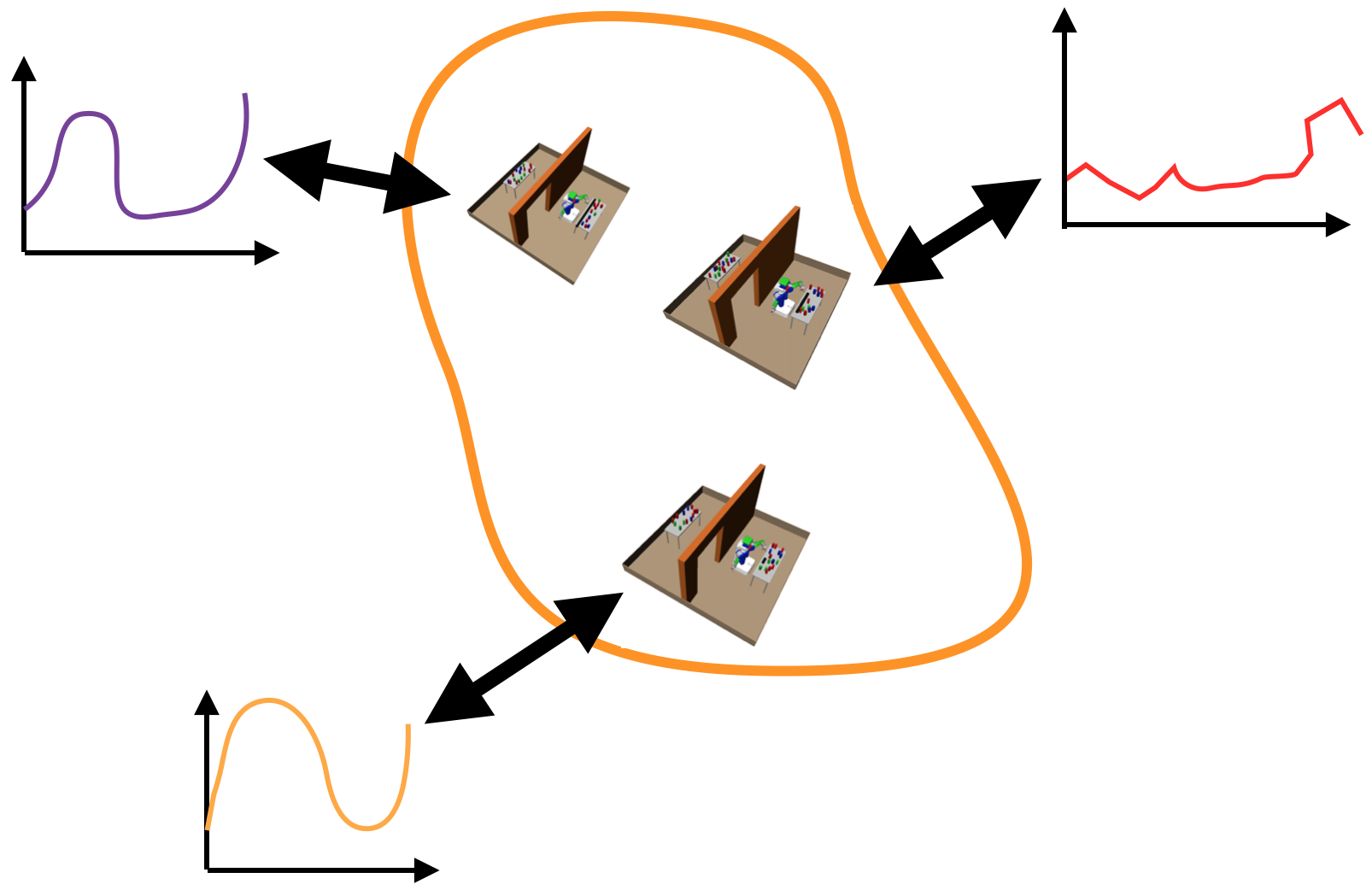}
\caption{An illustration of a score-space representation of
three different problem instances. We assume an one-dimensional 
constraint for the illustration purpose. The x-axis of
a score function represents different values of constraints,
and y-axis represents the score of a plan that conforms to a 
constraint.}
\label{fig:score_space_illustration}
\end{figure}

  The main advantage of the score-space representation
is that it gives direct information about similarity between problem instances and 
solution constraints, without depending on a hand-designed representation.
Similar problem instances will have similar patterns of scores on different
constraints, and similar constraints will have similar patterns of scores
on different problem instances.

 For instance, consider again the task in Figure \ref{fig:grasp_domain}. 
By computing a plan associated
with the solution constraint that forces the robot to  approach the 
object from the top and
observing that it has a low score, we can learn that it is occluded from
above, e.g. by a shelf, and can predict other more appropriate action choices.
This information is not biased by the designer's choice of representation: having
learned that approaching the object from the side did not work,
a learning algorithm may try to approach the object from the top. Such reasoning
is much more difficult when the learning algorithm is forced
to reason in terms of a particular representation of the problem,
such as poses of obstacles or 2D or 3D images of the scene. 
The main disadvantage is that the score-space information is
computationally expensive to obtain, since it requires computing
the plans associated with solution constraints.

This observation brings us to the third challenge of
 learning to plan, which is how to transfer knowledge from past experience
to the current problem instance efficiently.  We propose to solve this challenge
by using the expectation and correlation of the scores of solution
constraints from past problem instances to determine which solution
constraint to try next. Our intuition is that the solution constraints
that performed well or poorly together are more likely to do so in a
new problem instance; for instance, in the previous example for
picking an object from a shelf, grasps from the sides would have worked
well in most of the problem instances, while grasps from the top or bottom
would have worked poorly.

Building on this intuition, we propose an algorithm that directly
reasons with the correlation information in the score-space 
representation of a problem instance. We
assume that a score vector that represents a problem instance is
distributed according to a Gaussian distribution, 
and propose~\boxalg, which is an upper-confidence-bound (UCB)
type, experience-based, black-box function-optimization
technique~\citep{MunosFTML14,SrinivasICML10}.~\boxalg~learns 
to suggests solution constraints based on its belief
about the scores of solution constraints, whose prior
distribution is computed using scores of constraints
in the past planning problem instances, and whose 
posterior distribution is defined by updating the parameters
of the Gaussian distribution using the score feedback from the environment.
The overall work-flow is shown in Figure~\ref{fig:box_overall}.
 
\begin{figure*}
\centering
\includegraphics[scale=0.35]{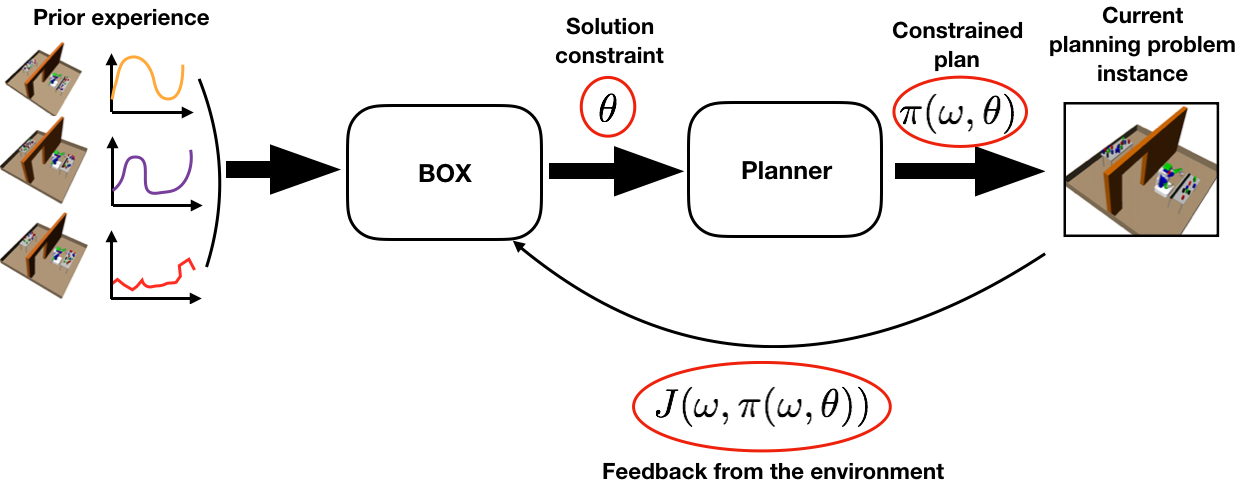}
\caption{Overall work-flow of BOX. It receives prior experience data
in the form of score values, which it uses to compute the prior belief
of scores of a current problem instance. It then suggests a constraint 
to the current problem instance, receives a score, and updates its
belief. This process is repeated until we find a solution.} 
\label{fig:box_overall}
\end{figure*}

This paper is an improved and extended version of our prior paper~\citep{KimICRA17}.
In particular, we make the following additional contributions. 
\begin{itemize}
\item We prove that our algorithm has a sublinear regret in number
of evaluations under
certain conditions, using the theoretical results from
Bayesian optimization literature~\citep{SrinivasICML10}. 
\item We propose a new algorithm 
that reduces the cardinality of an initial set of constraints. 
It reduces the worst-case evaluation time, the computation
time for updating the covariance matrix, and the number of
evaluations of a constraint. We add experiments to evaluate
the effect of this algorithm. 
\item We present a new simulation result that requires a significantly
longer planning horizon than the ones that we considered
in our prior work. Besides the longer horizon, this new
domain is especially hard because uniform random sampling
of actions would
frequently call a motion planner on infeasible problem 
instances. 
\end{itemize}

In all of the tasks, we show that~\boxalg~can accelerate planning 
significantly over a basic planner that does not use solution
constraints.  We also provide a comparison to a sampler that picks
solution constraints uniformly at random, and to a state-of-the-art black-box
function optimization technique called deterministic optimistic
optimization (DOO)~\citep{MunosNIPS11} that does not use the score space
representation. We find that~\boxalg~ outperforms these other methods.
Additional experiments on using a minimal constraint set built
using our greedy algorithm indicate that
we can almost completely eliminate the covariance matrix
inversion time for~\boxalg, and reduce the total planning time by
a factor of 5 for some domains.

\section{2. Related work}
\label{related}

There is a substantial body of work aimed at improving motion planning
performance on new problem instances based on previous experience on
similar problem
instances~\citep{BerensonICRA12,HodalEngMech08,HutchinsonCybernetics92,
jetchev2013,LienRSS09,PhillipsRSS12}.
The typical approach is to store a large set of solutions to earlier
instances so that, when presented with a new problem instance, one can
(a) retrieve the most relevant previous solution and (b) adapt it to
the new situation.  These methods differ in the way that they find the
most relevant previous solution and how it is adapted.

Several of these approaches define a similarity metric between problem
instances and retrieve solutions based on this metric.  For example,
~\cite{HodalEngMech08} use the distance between the start and
goal pairs as the metric, whereas,~\cite{HutchinsonCybernetics92}
 based their metric on descriptions of quickly generated low-quality 
solutions for the current and 
previous instance. \cite{jetchev2013}  use a mapping into a
task-relevant space and measure similarity in that space, 
with a learned metric.

Instead of defining solution similarity, \cite{BerensonICRA12} define
relevance of earlier solutions by measuring the degree of
constraint violation, for example collision, in the current situation.  A
related idea is developed by \cite{PhillipsRSS12}, where the search
graph of past solutions is saved and the search for the current
problem instance is biased towards the part of this past graph that is
still feasible.

For more complex robotic planning, such as for mobile manipulation in
cluttered environments, complete solutions are more difficult to adapt
to new problem instances.  In particular, the length of the plans is
highly variable and they contain both discrete and continuous
parameters.

Some earlier approaches have also
focused on predicting partial solutions, in the form of a goal state or
subgoals, instead of a complete solution. For instance, 
in the work of \cite{DraganISRR11}, the objective is to learn from
previous examples a classifier (or regressor) that, given a
hand-designed feature representation of a planning problem instance,
enables choosing a goal that leads to a good locally optimal
trajectory.  In the approach of \cite{FinneyRSS07}, 
the goal is to learn a model that
predicts partial paths or subgoals, from a given parametric
representation of a planning problem instance, aimed at enabling a
randomized motion planner to navigate through narrow passages.

Several approaches have been proposed for learning representations
for robot manipulation skills with varying set of objects~\citep{KroemerISER16,KroemerICRA17}.
 Specifically,~\cite{KroemerISER16} proposes
a feature selection method that, given a certain basic features of objects
in a scene, such as bounding boxes describing their shapes, predicts
relevance of each feature for the given task. The relevant features 
are then used with a low-level robot motion skill represented with 
a dynamic motion primitive. While this approach is quite general, it
still requires a human to design the basic set of features that
are relevant for the tasks that the robot will encounter. 

\cite{ZhuICCV17} proposed an approach to learn a representation
for learning a complete policy from RGB images. Authors designed an architecture
and defined a set of discrete high-level actions that allows an agent
to accomplish different tasks. Similar to our approach, they
learn a representation from planning experience. However,
they learn a complete policy, whereas our approach
learns to predict solution constraints. Moreover, their method 
does not take a robot into account - for manipulation tasks, however,
visual representation greatly suffers from occlusion by the end-effectors
of the robot. Our score-space representation  abstracts away
from this problem by directly representing a planning scene with
how well the predicted constraints works.

Our approach can be seen as a method for choosing actions from a
library; several methods have been proposed for this
problem. \cite{DeyAAAI2012} propose a method that finds a fixed ordering of
the actions in a library that optimizes a user-defined submodular
function, for example, the probability that a sequence of candidate
grasps will contain a successful one. Unlike our work, this method
produces a static list, which does not change across different problem
instances.  Later, \cite{DeyRSS2012}, generalized the
approach by producing an ordered list of classifiers (operating over
environment features) that select actions for a given problem
instance.  This approach again requires hand-designed features for the
problem instances.

We formulate our problem as a black-box function optimization problem.
In particular, \boxalg~is motivated by the \emph{principle of 
optimism in the face of uncertainty},
 which is well surveyed by 
\cite{MunosFTML14}. The main idea is 
to select the most ``optimistic'' item from the
given set of items, by constructing an upper bound on the values of
un-evaluated items. The performances of these algorithms are
heavily influenced by how the upper-bounds are constructed. 
For instance, \dooalg~(Deterministic Optimistic Optimization), developed
by~\cite{MunosNIPS11}, 
first constructs the upper bound on the target function using
manually specified semi-metric and a smoothness assumption on the 
target function. It then chooses the next point to evaluate based 
on this upper bound in order to balance exploration and exploitation.

Another suite  of algorithms for black-box function optimization problems are
 Bayesian optimization algorithms~\citep{SrinivasICML10,ZiAISTATS16,SnoekNIPS12}.
In Bayesian optimization, an upper bound of the target function is constructed
based on a hand-designed covariance matrix and the Gaussian process assumption.
For instance, \cite{SrinivasICML10} constructs an upper bound using
the confidence interval in an algorithm called Gaussian Process Upper Confidence Bound (\gpucb). 
At each time step, \gpucb~evaluates the point that has the 
highest UCB value, observes the function
value, and updates the Gaussian process. It returns the query point that resulted
in the highest value.

The problem with these approaches is that, in general, they require a human
to design a similarity function on problem instances and planning
solutions: \dooalg~requires
 a semi-metric, and Bayesian optimization algorithms require a kernel.
This introduces human bias in constructing the upper-bounds on the
target function, which strongly influences the performance of these
algorithms. Moreover, as we have shown in the introduction, designing
such a function is not trivial: even a small change in a problem instance 
 may induce a large change in scores.

Rather than relying on a hand-designed similarity function over problem instances
and planning solutions, our algorithm,~\boxalg, operates in a 
vector space of scores of possible solutions where the correlation
information among them can be computed easily.

\section{3. Problem Formulation}
Our premise is that calling the planner with a solution constraint,
although much more efficient than the completely unconstrained 
problem, takes a significant amount of time and may generate
significantly suboptimal plans if the solution
constraint used is not a good match for the problem instance.
Given a new instance we will call the planner with a fixed number of
solution constraints and return the best plan obtained. So, our
problem is which solution constraints should be tried,
and in what order. 

We formulate the problem as a black-box function optimization problem
over a discrete space of candidate solution constraints, and use
upper bounds constructed from experience on previous problem instances
as well as the accumulated experience on this instance to determine
which constraint to try next.

Formally, we have a sample space for problem instances, $\probspace$, whose
elements $\prob$ are distributed according to $P(\prob)$; a space
of possible planning solutions, $\solspace$; and a space of solution
constraints. The plan solution space includes all possible
assignments to all of the decision variables, and the solution constraint space
includes all possible assignments to a \emph{subset} of decision variables
for the given planning problem. The function $J(\prob,x)$ specifies the score of a
solution $x \in  \solspace$ on problem instance $ \omega \in \probspace$.
We assume a planner $\pi: \probspace \rightarrow \solspace$ that, given a problem
instance $\prob$ can return a solution $\pi(\prob) \in \solspace$ that is either
feasible or near-optimal depending on the nature of the problem. 
In addition, we assume that, given a solution constraint $\solcon$, the
planner $\pi$ will return $\pi(\prob,\solcon) \in \solspace$,which is a plan
subject to the solution constraint $\solcon$;  in general this
solution will not be optimal (so in general $J(\prob,\pi(\prob)) > J(\prob,\pi(\prob,
\solcon))$), unless $\solcon$ was perfectly suited to the problem
instance $\prob$, but constraining the plan to satisfy $\solcon$ will
make it significantly more efficient to compute. With a slight abuse of
notation, we will denote $J(\prob,\solcon) =  J(\prob,\pi(\prob,
\solcon))$, the evaluation of the solution constraint.

Let  $\apprxsolconspace= \{\solcon_1, \ldots, \theta_m\}$
be a set of samples from the space of solution constraints. 
Now, we formulate our problem 
as follows: given a ``training set'' of example problem instances
$\prob_1, \ldots, \prob_n$ sampled identically and independently from $P(\prob)$, 
a discrete set of solution constraints 
$\hat{\solconspace}$,
and the score function $J(\cdot,\cdot)$, generate a high-scoring 
solution to the ``test'' problem instance $\testprob$. 

An interesting problem that we do not explicitly address in this paper
is how to select the subset of decision variables for specifying 
constraints $\theta$. In this paper, we manually choose
a subset of decision variables as solution constraints, and
take the simple approach of solving the training problem 
instances and then extracting 
the $\theta$ values corresponding to these chosen constraints.
The details of constructing $\apprxsolconspace$
are provided in Algorithm~\ref{alg:GenTrainData}.

\subsection{3.1 Black-box function optimization with experience}
Instead of designing a problem-dependent representation for
problem instances, we represent a problem instance with 
a vector of scores of solution constraints
$\apprxsolconspace$, where
\begin{align*}
\Phi(\prob) = [J(\prob,\theta_1),\cdots,J(\prob,\theta_m)]
\end{align*}

$\Phi(\prob)$ here is a random vector that maps a sample
from the sample space of problem instances to $\mathbb{R}^m$. 
Using this representation, our training
data constructed from $n$ problem instances can be 
represented with a $n \times m$ matrix
$$\mathbf{D}=
\begin{bmatrix}
& \Phi(\omega_1) & \\
& \Phi(\omega_2) & \\
& \vdots & \\
& \Phi(\omega_n) &
\end{bmatrix}
$$
that we call the \emph{score matrix}.
Now, given a new problem instance, $\prob$, our goal is
to take advantage of one or more solution constraints 
in $\apprxsolconspace$ to find a high scoring plan
without evaluating all of solution constraints in 
$\apprxsolconspace$. To do this, 
we will develop a procedure that evaluates $J(\prob,\solcon)$ by computing
a plan $\pi(\prob,\solcon)$ for $k << m$
values of $\solcon$.  

We begin by making use of the intuition that some solution constraints
(via the plans they generate) are inherently more useful than others,
 independent of the problem instance. This leads to a naive score-space
approach, \staticalg, that tries solution constraints in $\apprxsolconspace$ 
in a static order according to the empirical mean scores in the $1\times m$ vector
computed by
\begin{align}
\hat{\mu}=\frac{1}{n}\sum_{i=1}^{n}\mathbf{D}_i
\label{eqn:mu}
\end{align}
where $i$ indicates the row of the score matrix,
and then returns the highest scoring plan obtained from trying
the top $k$ solution constraints.

This simple approach does not take advantage of the fact 
that there are correlations among the scores of solution constraints
across problem instances; that is, the score of a solution constraint
that has been already tried on this problem instance
can inform us about the scores of other untried but correlated
solution constraints. In order to exploit correlation, we assume that 
the random vector $\Phi$ is distributed according to a 
multivariate Gaussian distribution, $\mathcal{N}(\mu,\Sigma)$.

Now the score matrix is used to estimate the parameters of the 
prior distribution of $\Phi$, $\hat{\mu}$ and $\hat{\Sigma}$, 
 where $\hat{\mu}$ is defined in equation $\ref{eqn:mu}$ and 
\begin{align}
\hat{\Sigma} = \frac{1}{n-1}\sum_{i=1}^{n} (\mathbf{D}_i - \hat{\mu})^T (\mathbf{D}_i  - \hat{\mu})
\label{eqn:sigma}
\end{align}
is a $m \times m$ covariance matrix.
This prior distribution is updated given evidence about a 
new problem instance, in the form of score values.
Algorithm~\ref{alg:BOX} contains detailed pseudo-code for an algorithm
based on these ideas, called \boxalg, which stands for Blackbox
Optimization with eXperience.  It takes as input: $\testprob$, the ``test''
planning problem instance; $\zeta$, a constant governing the
magnitude of the exploration; $k$, the number of solution constraints to
evaluate; $\apprxsolconspace$, the set of solution constraints in the
training set;  $\hat{\mu}$ and $\hat{\Sigma}$, the parameters for
the prior distribution of $\Phi(\testprob)$;  $J$, the scoring function; and $\pi$,
the planner. 

The algorithm first estimates the parameters of the prior distribution.
Then, it iterates over solution constraints: first it selects a 
solution constraint, then it uses the chosen constraint to construct a new plan, and
the score of that plan combined with the prior computed from 
$\mathbf{D}$ is used to determine the next solution
constraint to evaluate. 

We will use  $\Theta_t$
to denote the constraints that have been tried up to time
$t$, $\bar \Theta_{t}=\Theta \setminus \Theta_{t}$ 
to denote the ones not tried, $\solcon^{(t)}$ to denote
 the index of the solution constraint chosen at time $t$, 
$x^{(t)}$ to denote the associated plan, $J^{(t)}$ to denote the
 score of that plan on the given problem instance, and
$J^{1:t}$ and $J^{\overline{1:t}}$ to denote the scores of tried and untried
solution constraints up to time $t$, respectively.

We will use this constraint notation to refer to
corresponding rows and columns in the mean vector and empirical covariance matrix. For 
instance, at time $t$, we can
rearrange the covariance matrix $\hat{\Sigma}$ as
\begin{align*}
\begin{bmatrix}
\hat{\Sigma}_{\bar{\Theta}_t,\bar{\Theta}_t} & \hat{\Sigma}_{\bar{\Theta}_t,\Theta_t}\\
\hat{\Sigma}_{\Theta_t,\bar \Theta_t} & \hat{\Sigma}_{\Theta_t,\Theta_{t}}\\  
\end{bmatrix}
\end{align*}
where the subscript represents a set of rows and columns of the matrix
$\hat{\Sigma}$. This way, the top-left block matrix is the covariance among
untried solution constraints, the top-right and bottom-left
represent covariance among tried and untried solution constraints,
and the bottom-right represents the covariance among the tried
solution constraints.

In line 1 of Algorithm~\ref{alg:BOX}, we first estimate the prior
distribution of the score function for a new problem instance $\prob_{n+1}$.
 For the consistency of notations we assume $\hmu^{(0)} = \hmu$ and $\hS^{(0)} = \hS$.
Line 2 selects the next
solution constraint to try based on the principle of
\emph{optimism in the face of uncertainty}, by selecting
the one with the maximum upper confidence bound (UCB).
The next three lines generate a plan using the chosen solution
constraint, and then evaluate it. At iteration $t$, 
given the experience of trying $\Theta_{t} =[\theta^{(1)},\cdots,\theta^{(t)}]$ 
and getting scores $J^{1:t} := [J^{(1)},\cdots,J^{(t)}]$, 
our posterior on the scores of the untried solution constraints,
denoted $J^{\overline{1:t}}$, is 
$$J^{\overline{1:t}} | J^{1:t}  \sim  
\mathcal{N}(\hat{\mu}^{(t)}_{\bar{\Theta}_t},~\hat{\Sigma}^{(t)}_{\bar{\Theta}_t,
\bar{\Theta}_t})$$
where
\begin{equation}
\begin{aligned}
\hat{\mu}^{(t)}_{\bar{\Theta}_t} &=  \hat{\mu}_{\bar{\Theta}_t} + 
\hat{\Sigma}_{\bar{\Theta}_t,\Theta_t}
(\hat{\Sigma}_{\Theta_t,\Theta_t})^{-1} 
( J^{1:t} - \hat{\mu}_{\Theta_t} )\\
\hat{\Sigma}^{(t)}_{\bar{\Theta}_t} &= 
\hat{\Sigma}_{\bar{\Theta}_t,\bar{\Theta}_t}
- \hat{\Sigma}_{\bar{\Theta}_t,\Theta_t}
(\hat{\Sigma}_{\Theta_t,\Theta_t})^{-1}
\hat{\Sigma}_{\Theta_t,\bar{\Theta}_t}
\end{aligned}
\label{eqn:update}
\end{equation}
The constant $\zeta$ governs the size
of the confidence interval on the scores. 
We show how the constant can be set through
theoretical analysis of regret bounds in 
the next section. 
 The number of evaluations $k$ should be
chosen based on the desired trade-off between computation time
and solution quality.

\begin{algorithm}[tb]
\small
   \caption{BOX($\testprob, C, k, \solconspace, \bf{D}, J, \pi$)}
   \label{alg:BOX}
\begin{algorithmic}[1]
\STATE Compute $\hmu^{(0)}$ and $\hS^{(0)}$ according to Eqns~\ref{eqn:mu} and~\ref{eqn:sigma}
\FOR{ $t=1$ {\bfseries to} $k$ }
\STATE $\theta^{(t)} = \argmax_{i \in \bar \Theta_t } \hat{\mu}_i^{(t-1)} 
+ \zeta\cdot\sqrt{\hat{\Sigma}^{(t-1)}_{ii}} $
\emph{// $i^{th}$ entry  and $i^{th}$ diagonal entry   }
\STATE $x^{(t)} = \pi(\testprob,\solcon^{(t)})$
\STATE $J^{(t)} = J(\testprob,x^{(t)})$
\STATE Compute $\hat{\mu}^{(t)}$ and $\hat{\Sigma}^{(t)}$ using eqn. \ref{eqn:update}
\ENDFOR
\STATE $t^* = \argmax_{t\in\{1,\cdots,k\}} J^{(t)}$ 
\STATE{\bf return} $x^{(t^*)}$
\end{algorithmic}
\end{algorithm}

\begin{algorithm}[tb]
\small
\caption{GenerateTrainingData($n, \pi, J, [\prob_1,\cdots,\prob_n]$)}
\label{alg:GenTrainData}
\begin{algorithmic}[1]
\FOR{$\prob$ {\bfseries in} $ [\prob_1,\cdots,\prob_n] $} 
\STATE $x_i = \pi(\prob)$ \emph{// repeat to get multiple solutions if desired}
\STATE $\solcon_i = {\it extractConstraint}(x_i)$ {\emph // elements of $\solconspace$}
\ENDFOR

\FOR{$\prob$ {\bfseries in} $ [\prob_1,\cdots,\prob_n] $} 
\FOR{$\solcon$ {\bfseries in} $\apprxsolconspace$}
\STATE $J(\prob,\solcon) = J(\prob,\pi(\prob, \solcon))$ \emph{// elements of $\mathbf{D}$}
\ENDFOR
\ENDFOR
\STATE {\bf return} $\mathbf{D}, \solconspace$
\end{algorithmic}
\end{algorithm}

In order to create the score matrix $\mathbf{D}$ and
solution constraints $\apprxsolconspace$, we run Algorithm $\ref{alg:GenTrainData}$.  This
algorithm takes as input $n$, the number of training problem
instances, $\pi$ a planning algorithm that can solve problem instances
$\testprob$ without additional constraints, $J$, the scoring function
for a plan, and $ [\prob_1,\cdots,\prob_n]$, a set of training sample 
problem instances drawn iid from $P(\prob)$.  For
each problem instance, a solution is generated using $\pi$, and a
constraint is extracted from the solution and added to set
$\apprxsolconspace$.  The process of extracting constraints is
domain-dependent; several examples are illustrated in
the experiment section.  Each new solution constraint
is used to generate a solution $\pi(\testprob,\solcon)$ whose
score $J(\testprob,\solcon)$ is stored in the $\mathbf{D}$ matrix.

\subsection{3.2 Illustrative examples}
We now provide concrete examples of running \boxalg~on some simple
examples. Suppose that our constraint is a set of four different grasps,
defined by approach vectors, from the top, left, bottom, and right. Our
problem is to plan a collision-free path to grasp a target object.
The planner is constrained to use the chosen grasp approach direction
 to grasp the target object. We arbitrarily choose $\zeta=1.96$ to ensure 
$95\%$ confidence interval in our experiments, but it could be 
tuned via cross validation for
better performance.

In Figure~\ref{fig:score_matrix_ex}, we show an example of a score matrix 
$\expmat$ obtained by running~Algorithm \ref{alg:GenTrainData}. In this figure,
the target object is represented with a black circle, and the blue
rectangular objects represent obstacles. We have four training problem
instances, shown across the rows of the score matrix, and the four constraints
across the columns. For illustrative purposes, we assume a 
simple binary score function, which
outputs one if a constraint is feasible for the given problem instance, and
zero otherwise. For example, for the first training problem instance, 
the top-approaching direction is feasible because there is 
no obstacle blocking the object
in that direction, whereas the left-approaching 
direction is blocked with an obstacle. 
For a such binary score function, other prior assumption on the target function,
such as Bernoulli distribution, might be more suited; however, we will in general
consider score functions that take on real numbers, as we will 
demonstrate in our experiment section.

In Figure~\ref{fig:sigma_ex}, we show the result of computing the 
covariance matrix $\hat{\Sigma}$ using $\expmat$ and equation \ref{eqn:sigma}. 
In order to understand \boxalg~more thoroughly, we note 
some salient correlation information  in $\hat{\Sigma}$. First, the top-approaching
constraint is positively correlated with the right-approaching direction,
whereas it is negatively correlated with the bottom-approaching direction.
So, in a new problem instance, if we find that the top-approaching constraint
fails, then it will increase the UCB value of the score for the bottom-approaching
direction while decreasing the UCB value of the right-approaching direction.

We will illustrate these types of behaviors of using two
 problem instances shown in
Figure~\ref{fig:two_pinsts}, where the task is to plan a collision-free
path to grab the circular magenta object, which is occluded by
red obstacles. Clearly, for the problem 
shown in the first row, the only constraint that would work
is the bottom-approaching direction. 
For the second problem instance, the left-approaching direction would be
the only feasible constraint.

Figure~\ref{fig:evolution_1} shows the evolution of UCB values,
$$ \hat{\mu}_i^{(t)} + \zeta\cdot \sqrt{\hat{\Sigma}_{ii}^{(t)}}$$
of the different constraints we denoted with $i$, as \boxalg~suggests constraints and
receives feedback from the planner and the environment.
For example, from the score matrix shown in Figure~\ref{fig:score_matrix_ex},
we can see that the average values of the scores for top-,left-, and 
right-approaching directions are 0.5. The elements in the 
diagonal of the covariance  matrix, which are variances of the scores
of different constraints, are approximately 0.33 for these three constraints.
These give approximate UCB values of 0.83 for these constraints.

The first plot in Figure~\ref{fig:evolution_1} shows the UCB values 
when $t=1$. There is a tie among 
UCB values of the first, second, and the last constraints, so we randomly
break the tie and select the second constraint, marked with
the red circle. After trying to plan a
path with this constraint, we see that it is infeasible. The
second plot shows the updated UCB values  after
observing that the second constraint has a score of zero, using
Eqn~\ref{eqn:update}. We see that the UCB values of the first and
 the last constraint remained unchanged, while the third constraint
has increased. This is because the second constraint 
 has zero correlation with the first and last
constraints, but has negative correlation with the third constraint,
as shown in the Figure~\ref{fig:sigma_ex}. After this, the last and
the first constraints have the same UCB values, and we again randomly
break the tie; unfortunately, we chose the fourth one, but from
this we can update our UCBs such that all other constraints except
the third one are infeasible. 

This example is a particularly hard for~\boxalg, because from our prior 
experience the third constraint was feasible only 1/4 of the time
with the lowest variance. Therefore, our belief about its score was
quite low. We now consider the second problem instance
in Figure~\ref{fig:two_pinsts}, which is more favorable. 
Figure~\ref{fig:evolution_2} shows the evolution of UCB values. 
At $t=1$, shown in the first plot,
 we randomly break the tie, and chose the first constraint.
After observing this is infeasible, at $t=2$, the UCB value of the fourth
constraint, which has positive correlation with the first constraint,
also reduces to almost 0. The third constraint was negatively correlated with
the first constraint, so its UCB value increased; however, the second
constraint, which is the correct constraint for this problem instance,
has zero correlation with the first constraint, and its UCB value is
higher than the third one even after the update. Hence \boxalg~ends up
choosing the feasible constraint after just a single mistake.

\begin{figure*}[htb]
\centering
\begin{subfigure}[b]{0.3\textwidth}	
\includegraphics[scale=0.2]{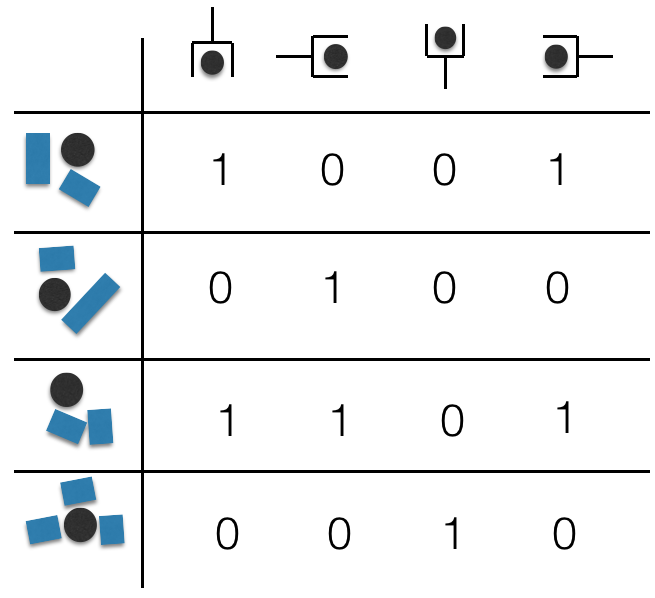}
\captionsetup{justification=centering,margin=1cm}
\caption{Score matrix, $\mathbf{D}$}
\label{fig:score_matrix_ex}
\end{subfigure}
\begin{subfigure}[b]{0.3\textwidth}	
\includegraphics[scale=0.2]{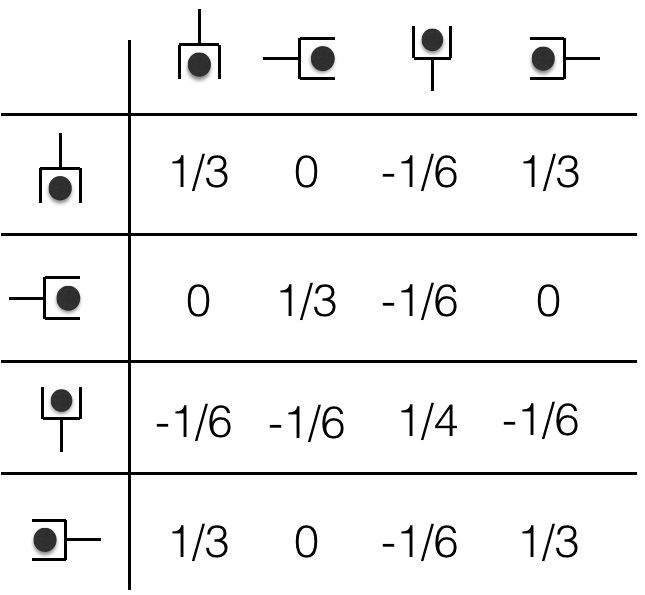}
\captionsetup{justification=centering,margin=1cm}
\caption{Covariance matrix, $\hat{\Sigma}$  }
\label{fig:sigma_ex}
\end{subfigure}
\begin{subfigure}[b]{0.3\textwidth}	
\includegraphics[scale=0.2]{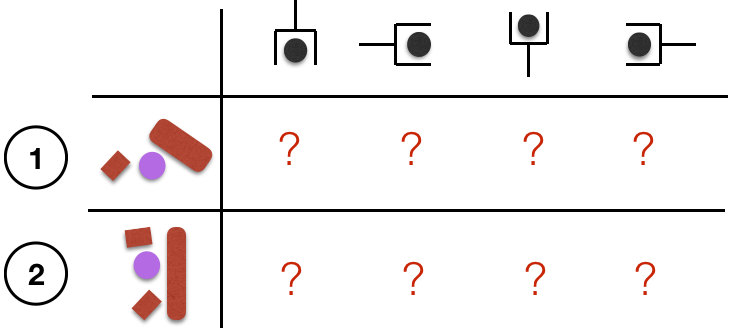}
\vspace{0.7cm}
\captionsetup{justification=centering,margin=0.5cm}
\caption{ Two new problem instances }
\label{fig:two_pinsts}
\end{subfigure}
\captionsetup{justification=centering,margin=0.5cm}
\caption{Score and covariance matrices for running \boxalg, and two new problem instances}
\end{figure*}

\begin{figure*}[htb]
\centering
\begin{subfigure}[b]{1\textwidth}	
\centering
\includegraphics[scale=0.2]{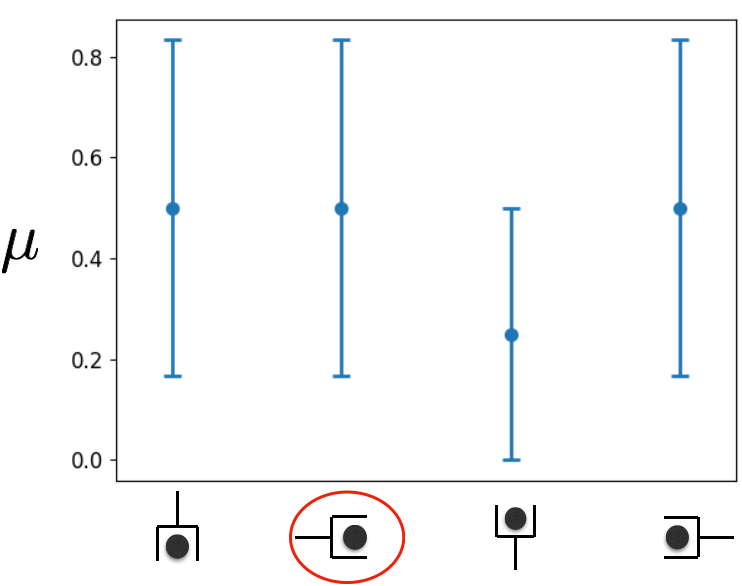}
\includegraphics[scale=0.2]{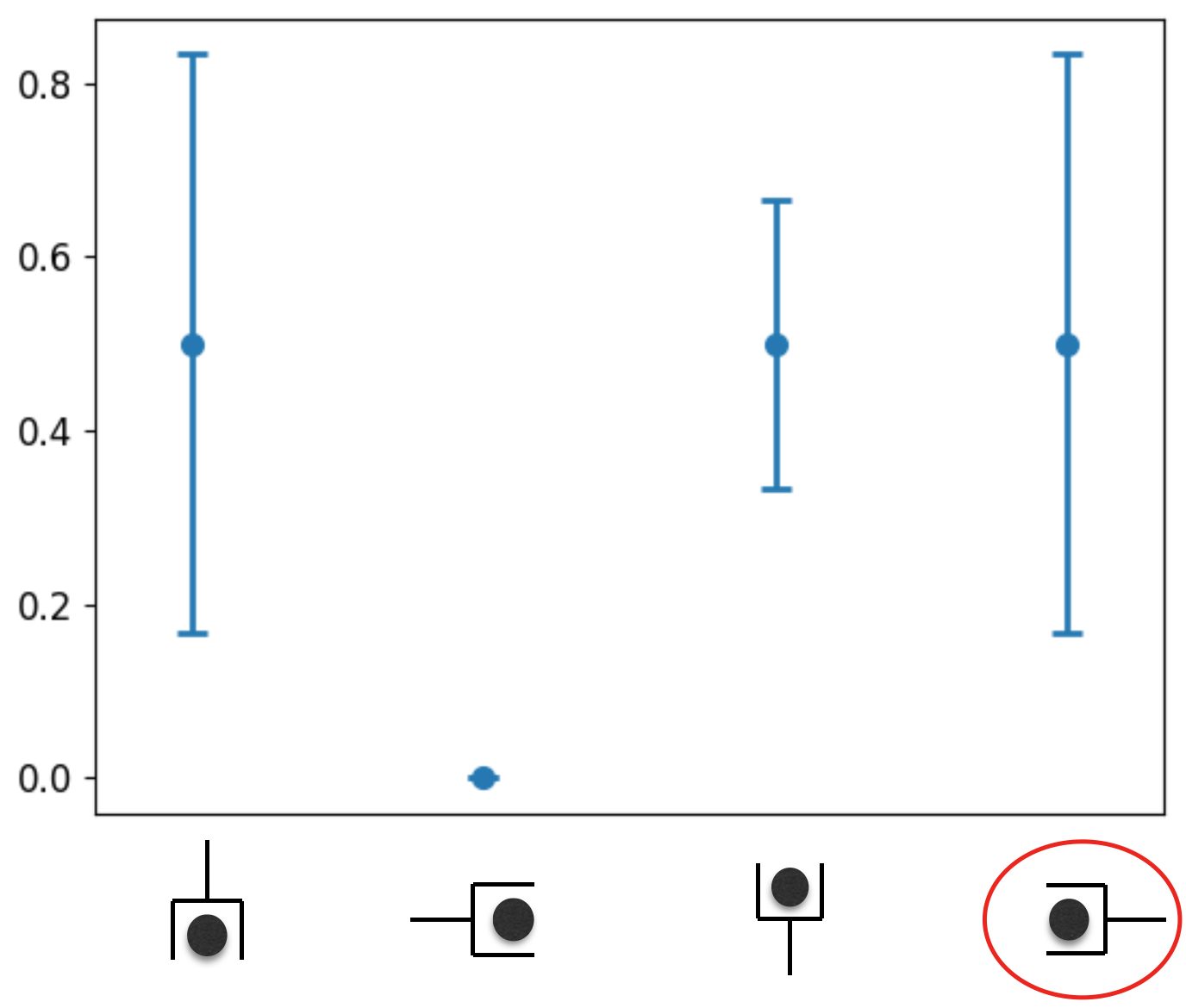}
\includegraphics[scale=0.2]{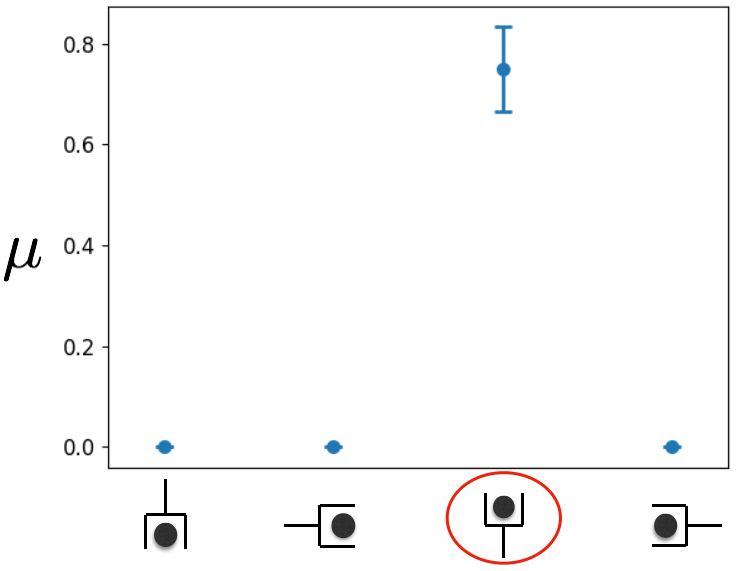}
\captionsetup{justification=centering,margin=2cm}
\caption{Evolution of $\mu$ and UCBs for the first problem instance}
\label{fig:evolution_1}
\end{subfigure} \\
\begin{subfigure}[b]{1\textwidth}	
\centering
\includegraphics[scale=0.2]{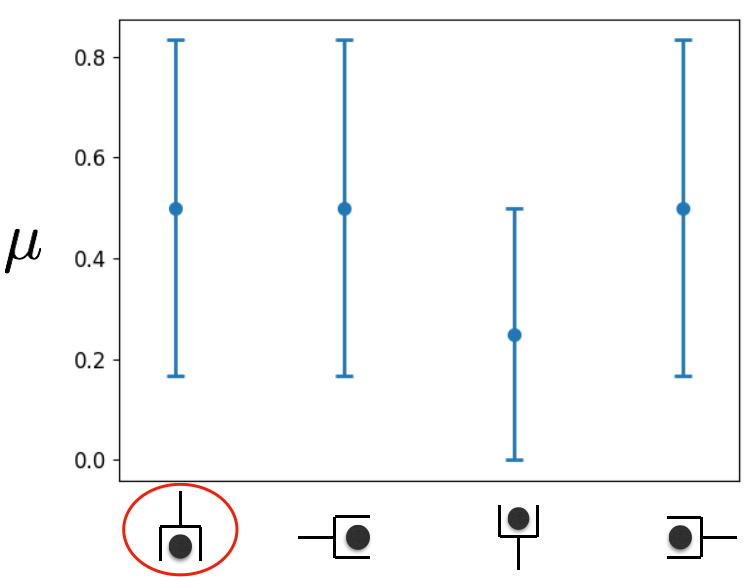}
\includegraphics[scale=0.2]{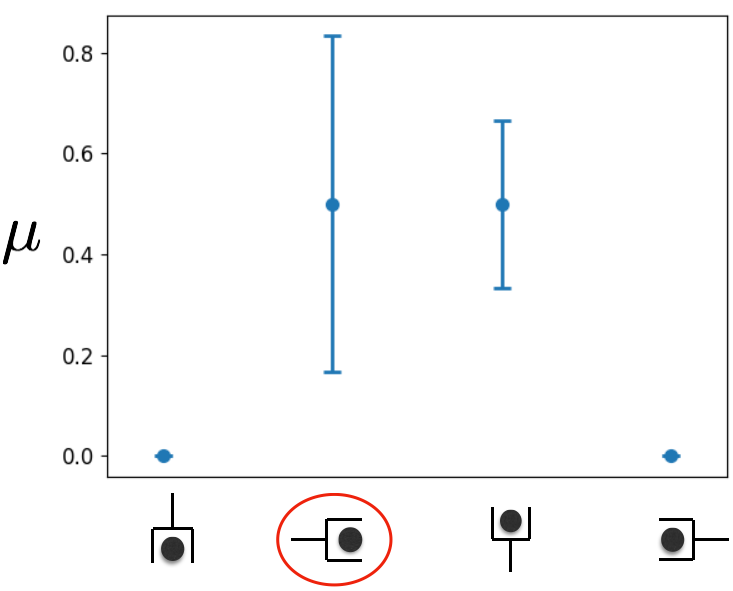}
\captionsetup{justification=centering,margin=2cm}
\caption{Evolution of $\mu$ and UCBs for the second problem instance}
\label{fig:evolution_2}
\end{subfigure}
\captionsetup{justification=centering,margin=0.5cm}
\caption{Illustration of how UCBs change as \boxalg~uses constraints in two
different problem instances}
\end{figure*}

\section{4. Theoretical analysis of \boxalg}
\label{sec:bound}
In this section, we analyze how the difference between the score of the
best constraint in our constraint set $\apprxsolconspace$ and 
the score of the best evaluated constraint changes over the
iterations of~\boxalg~for a given problem instance $\prob$. 
We first describe our notation and 
assumptions.

Since we focus on analyzing \boxalg~for a \emph{new} $\prob$, 
with slight abuse of notation we denote the scores of all the
 constraints on problem instance $\prob$ as 
$J= [J(\prob, \theta_i)]_{i=1}^m \in \R^m$. 
We assume that $J\sim \mathcal N(\hat\mu, \hat\Sigma)$. In other words, this assumption means that
 there exist a multi-variate Gaussian parameterized by $\hat \mu$ and 
$\hat\Sigma$ such that $J$ is a sample drawn 
from $\mathcal N(\hat\mu, \hat\Sigma)$. 
The Gaussian parameters $\hat \mu$ and $\hat \Sigma$ 
are defined in in Eq.~\eqref{eqn:mu} and 
Eq.~\eqref{eqn:sigma}. 
We use the shorthand $\hat\Sigma_{A}$ to denote the submatrix
$\hat\Sigma_{A,A}$ for any subset of constraints $A\subset \Theta$. 

We are also going to assume that the diagonal terms of $\hat\Sigma$ are  
bounded, meaning that there exists a constant $c>0$ such 
that $\hat\Sigma_{\theta} < c, \forall \theta\in \Theta$. 
This assumption limits the variance of the scores of 
each constraint to be finite. It is a valid assumption 
in practice because the scores themselves are typically 
bounded both above and below.

We use \emph{regret} as the performance measure for~\boxalg, as typically 
done in the Bayesian optimization literature. The \emph{regret} of a 
black-box function maximization algorithm is defined as the difference between the 
score of the best constraint selected 
within the given budget $k$ and the optimal score $J_{\optsolcon}$ evaluated at the best constraint $\optsolcon$ in the pre-built 
finite constraint set  $\Theta$; that is
$$r_k = J_{\optsolcon} - \max_{t\in[k]} J^{(t)},$$
where $[k] = \{1,2,\cdots,k\}$ for any positive integer $k\leq m$.

\subsubsection{Theoretical results}
\label{ssec:thm}

We first state our main theorem, and then explain its implications and the proof strategy. 
\begin{thm} \label{thm:regret}
Pick $\delta\in(0,1)$ and $\sigma > 0$ such that $\hat \Sigma - \sigma^2\mI$
 is positive semi-definite. Assume that for $c>0,\ \hat{\Sigma}_\theta < c,\ \forall \solcon \in \apprxsolconspace.$ 
Then with probability at least $1-\delta$, 
the regret of~\boxalg~with $\zeta = (2\log \frac{1}{\delta})^\frac12$ satisfies
$$ r_k  \leq   2 \sqrt{2\log(\frac{1}{\delta}) \left(\frac{2(c-\sigma^2)\rho_k}{k\log(c \sigma^{-2})}  + \sigma^2\right)} ,$$
where $\rho_k = \max_{A\subseteq \Theta, |A|= k} \frac12\log\det(\sigma^{-2} \hat\Sigma_{A} )$.\end{thm}

On a high level, Theorem~\ref{thm:regret} shows that \boxalg can achieve almost zero regret with high probability, under some mild assumptions on the score vector $J$ and its distribution. As long as the score vector $J$ is a sample from $\mathcal N(\hat \mu, \hat\Sigma)$ and $\hat \Sigma$ has  decaying spectrum, \boxalg is guaranteed to converge to a constraint whose score is almost the same as the best possible score. We explain in more details below. 

In the proof of this theorem, we use a data augmentation trick~\citep{van2001art} 
that adds an auxiliary random variable $f\in \R^{m}$ to the graphical model of $J$. 
Fig.~\ref{fig1} illustrates our problem setup where the generative model is 
simply $J\sim \mathcal N(\hat\mu, \hat\Sigma)$. Fig.~\ref{fig2} illustrates 
the new augmented graphical model and the relation between $f$ and $J$, 
where the new generative model is,
\begin{enumerate}
\item Draw $f\sim \mathcal N(\hat\mu, \hat\Sigma - \sigma^2\mI)$;
\item Draw $J\sim \mathcal N(f, \sigma^2\mI)$.
\end{enumerate}
Notice that we can integrate out $f$ in Fig.~\ref{fig2} and arrive at the same distribution of $J$ as Fig.~\ref{fig1}, $J\sim \mathcal N(\hat\mu, \hat\Sigma)$.

\begin{figure}[h]
\centering
    \begin{minipage}{.4\textwidth}
\centering
\begin{tikzpicture}
\tikzstyle{main}=[circle, minimum size = 6mm, thick, draw =black!80, node distance = 8mm]
\tikzstyle{para}=[circle, minimum size = 5pt, inner sep=0pt]
\tikzstyle{connect}=[-latex, thick]
\tikzstyle{box}=[rectangle, draw=black!100]

  \node[para, fill = black!100] (alpha) [label=below:${\hat\mu,\hat\Sigma}$] { };
  \node[main, fill = black!10] (theta) [right=of alpha,label=below:$J$] { };

  \path (alpha) edge [connect] (theta);
\end{tikzpicture}
\caption{The graphical model of $J$  with Gaussian parameters $\hat\mu$ and $\hat\Sigma$.}
\label{fig1}
\end{minipage}
\hspace{1em}
\begin{minipage}{0.4\textwidth}
\centering
\begin{tikzpicture}
\tikzstyle{main}=[circle, minimum size = 6mm, thick, draw =black!80, node distance = 8mm]
\tikzstyle{para}=[circle, minimum size = 5pt, inner sep=0pt]
\tikzstyle{connect}=[-latex, thick]
\tikzstyle{box}=[rectangle, draw=black!100]

  \node[para, fill = black!100] (alpha) [label=below:${\hat\mu,\hat\Sigma-\sigma^2\mI}$] { };
  \node[main] (theta) [right=of alpha,label=below:$f$] { };
  \node[main, fill = black!10] (f) [right=of theta,label=below:$J$] { };
    \node[para, fill = black!100] (sigma) [right=of f, label=below:${\sigma^2}$] { };
  \path (alpha) edge [connect] (theta)
  (theta) edge [connect] (f)
    (sigma) edge [connect] (f);
\end{tikzpicture}
\caption{The augmented graphical model of $f$ and $J$ with Gaussian parameters $\hat\mu$, $\hat\Sigma$ and $\sigma$. The free parameter $\sigma>0$ is chosen such that $\hat \Sigma - \sigma^2\mI$
 is positive semi-definite.}
\label{fig2}
\end{minipage}
\end{figure}

Then, $\rho_k$ can be interpreted as the maximum mutual 
information gain of observations of size 
$k$~\citep{SrinivasICML10}: $$\rho_k =  \max_{A\subseteq \Theta, |A|= k} I(J_{A}; f_{A}).$$ We make the relation clear in Lemma~\ref{lem:rho}. In particular, the mutual information gain is closely related to the predictive variances of the observations.
\begin{lem}
\label{lem:rho}
Let $A \subseteq \Theta$, $|A|=k$ and suppose $\hS-\sigma^2\mI$
 is positive semi-definite. Then the mutual information between 
the augmented variable $f_A$ and the observation $J_A$ satisfy
$$I(f_A;J_A) = \frac12 \log\det(\sigma^{-2} \hat\Sigma_{A}),$$
and the information gain of the observations selected by \boxalg~satisfies
$$I(f_{\Theta_k};J_{\Theta_k}) = \frac12 \sum_{t=1}^k\log(\sigma^{-2} \hat\Sigma^{(t-1)}_{\theta^{(t)}}).$$
\end{lem}

The exact quantity of  $\rho_k$ depends on $\hat \Sigma$ 
and how fast the spectrum of $\hat \Sigma - \sigma^2\mI$ decays. For example, if the 
off-diagonal terms in $\hat\Sigma$ are all 0, there is no 
correlation between constraints, in which case 
$\rho_k=O(k)$. Theorem 1  shows that the regret is bounded by $O(1)$ which
does not decrease as $k$ increases. This is because the scores of all 
the constraints are independent, so that when some 
constraints are evaluated, there is no information gained about 
the scores of other constraints. For robotic domains that we consider,
 constraints are usually correlated: a constraint on
the approach vector of a grasp for an object, for example, would have
similar scores across different problem instances if the directions are
similar.

Suppose that our original constraint space has dimension $d$; that means, for any constraint $\theta\in \Theta$, it holds that $\theta \in \mathbb{R}^d$ and $\Theta \in \R^{m\times d}$. If the covariance matrix satisfies 
$\hat\Sigma\propto \Theta \Theta\T + \sigma^2\mI$,  we have $\rho_k=O(d\log(k))$ by~\cite[Theorem 5]{SrinivasICML10}. That means our regret bound yields
$$r_k \lesssim O\Big(\sqrt{\frac{d(\log k)^2}{k}}\Big)$$ 
which decreases toward 0 as $k$ increases. We use ``$\lesssim$'', approximately less than,  to reflect that the bound on $r_k$ also depends on a free parameter $\sigma>0$, which can be very small but is always non-zero.

The proof of Theorem~\ref{thm:regret} depends on the proof 
of~\cite[Theorem 3.1]{SrinivasICML10}, which shows that the 
cumulative regret of GP-UCB with noisy observations 
increases sublinearly in the number of evaluations.~\cite{SrinivasICML10}
assume that the diagonal terms of the covariance matrix are all  
smaller than 1 in their Theorem 3.1. However, in our problem formulation, 
the score function is noise-free. Moreover, the diagonal terms of 
our covariance matrix are not bounded by 1. 
We use the data augmentation trick~\citep{van2001art} 
that adds the auxiliary random variable $f$ to introduce 
artificial noise to our observations. More importantly, 
we adapt the artificial noise to the scale of the 
covariance matrix and the number of evaluations. 
Then, the proof strategy of Srinivas et al.  
can be reshaped to prove Theorem~\ref{thm:regret}.

The additional artificial noise is critical to the proof of Theorem~\ref{thm:regret} 
because it enabled the mutual information to bound the regret. If the artificial noise is 
0, the mutual information is not well bounded. The proofs of both Theorem~\ref{thm:regret} and 
Lemma~\ref{lem:rho} can be found in the appendix.
\label{ssec:explain}

\section{5. Constructing a minimal set}
In this section, we propose an algorithm for \boxalg~which tries to reduces
the cardinality  of $\apprxsolconspace$ while maintaining important properties,
such as the probability that the set will contain a constraint that is applicable
to a new problem instance.

We begin with the problem formulation.
We wish that, for all problem instances, there is at least one solution in our set. 
We define a constraint to be \emph{feasible} for a problem instance 
if we can find a solution that satisfies it.
We will say a constraint \emph{covers} a problem instance if it is feasible for a problem 
 instance. Given a set of constraints, $\apprxsolconspace$, we are interested in a 
minimum cardinality subset of the original constraint set that covers all the problem
 instances in the training data.
We will call such subset a \emph{minimal set}, and denote it with $\minset$.

Out of all the minimal sets, we are interested in those whose probability 
of success is maximized, so that the set can be applied to a wide range of problem instances.
 The probability of success is measured by
$$P(\minset\text{ succeeds on } \prob) = 1-\prod_{\solcon \in \minset}^{k}(1-p_\solcon) $$ 
where $k$ is the cardinality of the minimal set, and $p_\solcon$ is the probability of
 constraint $\theta_i$ in the minimal set being feasible. 
Notice that $p_\solcon$ can be approximated by empirical
counts of successes divided by the number of problem instances from our training data.
We will call a minimal set whose probability of success is maximum a
\emph{ maximally successful minimal set}.

Out of all the maximally successful minimal sets, we are interested in those 
that give maximum information on the values of the scores of other constraints,
in order to minimize the number of evaluations.
First note that the differential entropy $h$
of the multivariate Gaussian distribution $\mathcal{N}(\mu,\Sigma_{\minset})$ over a random vector
$\minset$
is defined by 
$$h(\minset) = \frac{|\minset|}{2}\Big[1+\log(2\pi)\Big] + \frac{1}{2} \log \det \Sigma_{\minset} $$
where $n$ is $|\minset|$ and $|\Sigma_{\minset}|$ is a determinant
of the covariance matrix $\Sigma_{\minset}$. Now, using this, we can define
the {\it gain function}, $g$, which characterizes the information
gain for scores of other constraints given an evaluation of a constraint $\theta_i$
\begin{align*}
g(\Sigma_{\minset},\theta_i)&= h( \minset ) - h( \minset | \theta_i) \\
&= \log( |\Sigma_{\minset}| ) - \log( |\Sigma_{\minset|\theta_i}| )
\end{align*}
where $|\Sigma_{\minset}|$ is the determinant of the covariance 
matrix of the minimal set $\minset$,
and $|\Sigma_{\minset|\theta_i}|$ is the determinant of the covariance matrix of the minimal
set after evaluating the constraint $\theta_i$ in $\minset$. We will call a maximally successful
minimal set that maximizes the gain function an \emph{optimally minimal set (OMS)} and denote it with $\minset^*$.

We now formulate the optimally minimal set construction problem as follows.
Given an original constraint set $\apprxsolconspace$, construct a minimal set
$\minset$ that maximizes the function
\begin{align*}
 c(\minset) = \sum_{\solcon \in \minset} p_\solcon + \lambda \cdot g(\Sigma_{\minset},\theta) 
\end{align*}
The first term is responsible for maximizing the probability of success for
$\minset$, and the second term is responsible for maximizing the sum of information
gain of each constraint in $\minset$. Now the optimally minimal set is defined as 
$$ \minset^* = \arg\max_{\minset \in 2^{\apprxsolconspace}} c(\minset) $$

Clearly, constructing $\minset^*$ is an NP-complete problem\footnote{We can
make a polynomial-time reduction to the minimum set-cover problem.}
This motivates us to devise a greedy approach that approximately optimizes
the function $c$, as described in Algorithm~\ref{algo:construct_minset}.

This algorithm takes as inputs: the experience matrix $\expmat$, 
the original constraint set
$\apprxsolconspace$, the approximated parameters of 
a distribution of score vectors $\hat{\mu}$ and $\hat{\Sigma}$,
the number of problem instances $n$ and the number of constraints 
of the original set $m$, and outputs
an approximation of an OMS, $L$. 

It operates by progressively constructing a list of constraints, 
$L$, using the gain function and probability
of success of a constraint, which is measured by the mean score function $\mu$.

It begins by first adding the index of the
 constraint that has the maximum mean score value. 
Then, it checks the number of problem instances covered by 
the current list of constraints $L$, 
using the helper function $nCoveredBy$. It then computes the uncovered
 problem instance indices, $U$.
The algorithm then computes the set of next candidate constraint 
to add to $L$, $\apprxsolconspace_{cand}$, by taking the 
constraint that maximally covers the currently 
uncovered indices $U$.  This maximal coverage step 
is to ensure that we are minimizing
the cardinality of $L$.

From this set, the algorithm updates $\apprxsolconspace_{cand}$ 
by considering only those constraints 
whose mean score is the maximum; it
is still a set, since there may be a tie in mean scores. From this set, the algorithm
chooses the next
constraint to add to $L$, by taking the one that has the maximum gain function value.
 We update the number of problem instances covered, and repeat until we cover
 all the problem instances.
Lastly the algorithm returns $L$, the set of constraints 
that approximates the OMS.
\begin{algorithm}[tb]
\small
\caption{ConstructOMS($\expmat,\apprxsolconspace,\hat{\mu},\hat{\Sigma},n, m$)}
\begin{algorithmic}
\STATE $\nextsolconidx = \arg\max_{\solcon \in \apprxsolconspace} \mu_\solcon$
\STATE $L = CreateList(\nextsolconidx)$ 
\STATE $C_{\max} = nCoveredBy(L,\{1,\cdots,n\}) $
\WHILE{ $C_{\max} \neq n$ }
\STATE $U = \{1,\cdots,n\})\setminus CoveredBy( L,\{1,\cdots,n\}) $ 
\STATE $\apprxsolconspace_{\text{cand}} = \{i | \max_{i\in{1,\cdots,m}} 
len(CoveredBy( \{L,i\}, U) ) \}$  
\STATE $\apprxsolconspace_{\text{cand}} = \{i|i = \arg\max_{i \in I_{\text{cand}}} \mu_i\} $  
\STATE $\nextsolconidx = \arg\max_{\solcon \in \apprxsolconspace_{\text{cand}}} 
g(\Sigma_{L},\theta)\} $ 
\STATE $L = \{L,\nextsolconidx\}$
\STATE $C_{\max} = nCoveredBy(L,\{1,\cdots,n\}) $ 
\ENDWHILE
\STATE {\bf return} $L$
\end{algorithmic}
\label{algo:construct_minset}
\end{algorithm}

\section{6. Experiments} 
\label{sec:experiments}
We demonstrate the effectiveness of score-space algorithms \staticalg~and 
\boxalg~in four robotic planning domains: grasp-selection,
grasp-and-base-selection, pick-and-place, and
conveyor-belt unloading. Each of these
domains has several decision variables and different types of
solution constraints. For all the problems, a problem instance 
varies in the sizes of the objects being manipulated, and the poses
of obstacles. 

In each of these domains, $\pi(\cdot,\solcon)$ finds values
of decision variables that are not specified in $\solcon$.
For example, for the grasp-and-base-selection domain $\pi(\cdot,\solcon)$ consists
of an IK solver and a path planner, and $\solcon$ specifies values such
as a robot base pose or a grasp for picking an object that is relevant
for achieving a goal. To implement~\rawalg, $\pi(\prob)$, we first uniformly
sample $\theta$ from their original space, such as $\mathbb{R}^2$ for
robot base pose, instead of from $\apprxsolconspace$, and then use
$\pi(\prob,\solcon)$ with the sampled constraints.

We are interested in both running time and solution quality. 
We compare score-space algorithms, \staticalg~and \boxalg, with 
\rawalg~as well as two other methods that
generate plans by selecting a subset of size $k$ of the 
solution constraints from $\apprxsolconspace$ 
and return the highest scoring one. As previously mentioned,
\staticalg~sequentially selects constraints based on their average
score values, without considering their correlation information.
{\randalg} selects $k$ of the $\theta_i$ values at random from $\hat{\Theta}$;
{\dooalg} is an adaptation of {\sc doo}~\citep{MunosNIPS11} to optimization of a 
black-box function over a discrete set, which is $\apprxsolconspace$ in 
our case.  Like \boxalg, it alternates between evaluating $\solcon_j$ and constructing upper
bounds on the unevaluated $\solcon_i$ for $k$ rounds. It assumes that the function is
Lipschitz continuous with constant $\lambda$, and uses the bound
$ J^\prob(\solcon_i) \leq J^\prob(\solcon_j) + \lambda\cdot l(\solcon_i,\solcon_j)$ for some
semi-metric $l$, $\lambda \in \mathbb{R}$. We use the Euclidean metric 
for $l$, and $\lambda=1$.

To show that score-space algorithms can work with different planners, 
we show results using two different planners: bidirectional
RRT with path smoothing implemented in OpenRAVE~\citep{openRave}, seeded with a fixed 
randomization seed value,  and Trajopt~\citep{SchulmanIJRR14}. 
In the pick-and-place and conveyor-belt unloading domains, where there is 
a narrow-passage path planning problem, Trajopt
cannot find feasible paths without being given a good initial solution, so
we omit it.

In each domain, we report the results using two plots, the first 
showing the time to find the first feasible solution and the second showing
how the solution quality improves as the algorithms are given more time.
Each data point on each plot is produced using leave-one-out
cross-validation.  That is, given a total data set of $n$ problem
instances and associated solutions, we report the average of $n$
experiments, in which $n-1$ of the instances are used as training data
and the remaining one is used as a test problem instance.

Grasp-selection, grasp-and-base-selection, and
pick-and-place problems are satisficing problems in which
we are mainly interested in finding a feasible solution. So, a binary
score function that specifies the feasibility of a given plan would be
sufficient to use~\boxalg. However, in many problems, we want to 
find a low-cost plan, rather than just a feasible one, using~\boxalg's
ability to seek optimal solutions from its library of constraints.

To do this, we design a score function that measures the trajectory 
length for a feasible plan, and that assigns a large cost 
if the plan is infeasible in the given problem instance. So, given
a plan $\pi(\prob)=(q_1,\cdots,q_l)$ where $q_i$ denotes a configuration
of the robot, our score function is
\begin{align}
J^\prob(x) = \begin{cases}
-\sum_{i=1}^{l-1} || q_{i+1} - q_{i} || \text{ if $x$ feasible in $\prob$}\\
d, \text{ otherwise}
\end{cases}\label{eq:dist_fcn}
\end{align}
where $||\cdot||$ denotes a suitable distance metric between configurations
and $d = min(\mathbf{D}) - mean(\mathbf{D})$. This is our
strategy for finding a domain dependent minimum score for failing to solve a problem.
Conveyor-belt unloading domain, on the other hand, is not a satisficing
problem: we are interested in maximizing the number of objects that a robot
packs into a tight room.  The Conveyor-belt unloading 
domain, is not a satisficing problem: we
are interested in maximizing the number of objects that the
robot packs into a tight room.
Therefore, naturally, our score function is defined
as the number of objects packed by a plan $\pi(\prob)$.

\subsection{6.1 Grasp-selection domain}
Our first problem domain is to find an arm motion 
to grasp an object that lies randomly either on a desk or a bookshelf,
 where there also are randomly placed obstacles. 
Neither the grasp of the object nor the final
configuration of the robot is specified, so the complete planning problem includes
choosing a grasp, performing inverse kinematics to find a pre-grasp
configuration for the chosen grasp, and then solving a motion planning problem
to the computed pre-grasp configuration.  

A planning problem instance for this domain is defined by an
arrangement of several objects on a table. Figure \ref{fig:grasp_domain}
shows two instances of this problem, which are also part of the
training data. There are up to 20 obstacles in each problem instance.
The robot's active degrees of freedom (DOF) are its left and right arms, 
each of which has 7 DOFs, and torso height with 1 DOF, for a total of 15 DOF. 
$\apprxsolconspace$ consists of 81 different grasps per
each arm, computed using OpenRAVE's grasp model function. $\solconspace$ 
would be all possible grasps for an object.
Notice that since our search space for solution constraints 
is discrete, \rawalg~is equivalent to \randalg.

Given a solution constraint $\solcon$, which is a grasp (pose of robot
hand with respect to the object) and an arm to pick the object with, 
it remains for $\pi(\prob,\solcon)$  to find an IK solution and motion plan,
which can be expensive, but predicting a good
grasp makes the overall process much more efficient. The
trajectory of the arm to the pre-grasp configuration, with the base
fixed, is scored according to eqn. \ref{eq:dist_fcn}, with a score of $d$
assigned to problem instances and constraints for which no
solution is found within a fixed amount of computation.

The experiments were run on a data set of 1800 problem instances.  
Figure \ref{fig:ff_arm_and_grasp} compares the
time required by each method to find the first feasible plan with
RRT as the path planner, and Figure \ref{fig:ff_arm_and_grasp_TrajOpt} compares
the time with TrajOpt as the path planner. In both of the plots,
we can observe that the score-space algorithms \staticalg~and \boxalg~outperform 
all other algorithms in terms of finding a good solution with a given
 amount of time. \boxalg~performs about three times
faster than \staticalg, showing the advantage of using the correlation information.
Compared to \dooalg~and \randalg, \boxalg~is more than nine times faster.
\dooalg~does only slightly better than \randalg, which illustrates that 
in the space of grasps, the Euclidean metric is not effective. 

Figure \ref{fig:arm_and_grasp_t} compares the solution quality 
vs time when RRT is used; figure \ref{fig:arm_and_grasp_t_TrajOpt}
compares the same quantities when TrajOpt is used.  Here, the score-space 
algorithms again outperform the other algorithms, with \boxalg~outperforming
\staticalg.
\begin{figure*}
\centering
	\begin{subfigure}[b]{0.31\textwidth}	
		\centering
		\includegraphics[width=\linewidth]{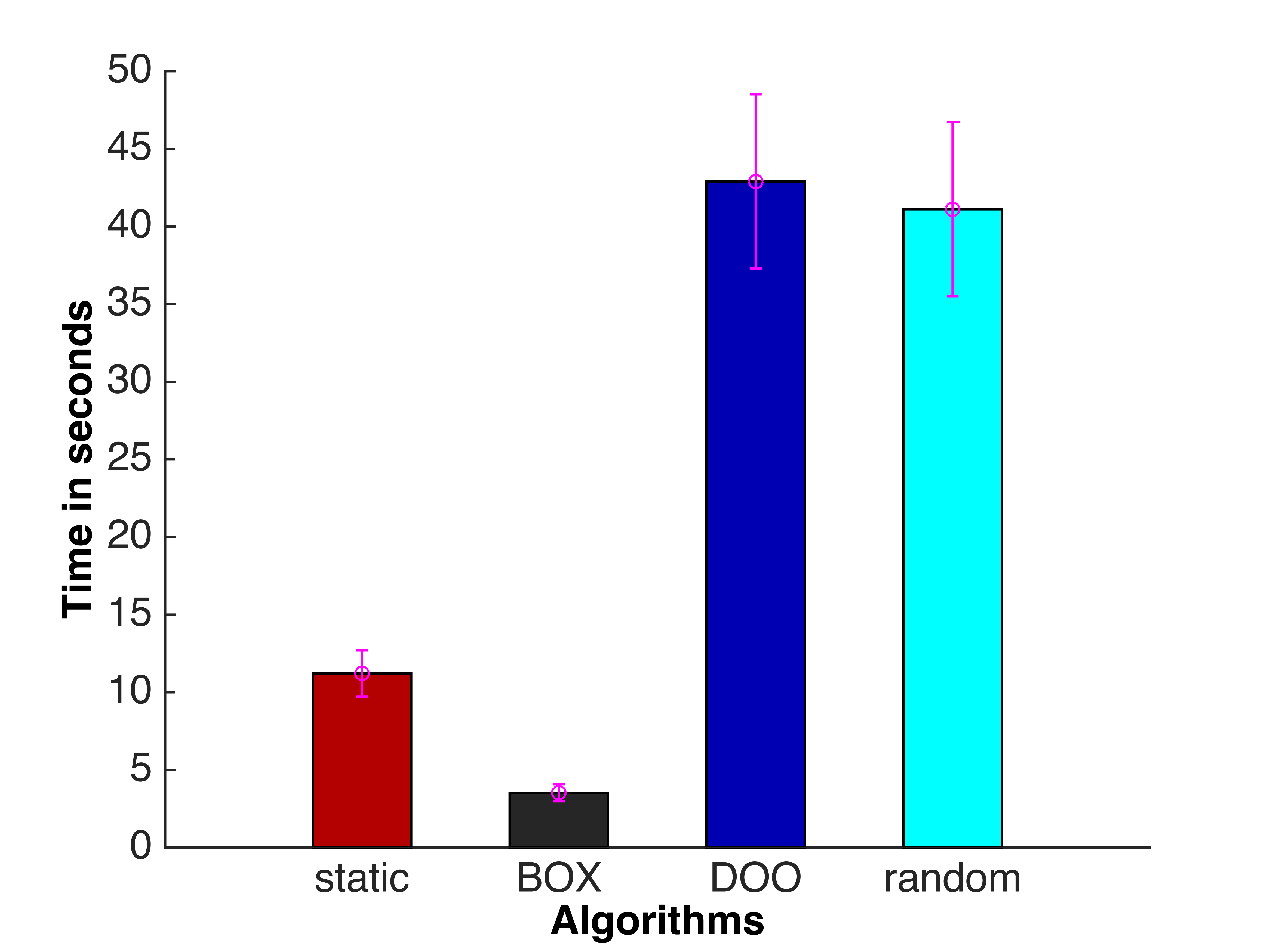}
		\caption{Grasp-selection (RRT)}
		\label{fig:ff_arm_and_grasp}
	\end{subfigure} 
	\begin{subfigure}[b]{0.31\textwidth}	
		\centering
		\includegraphics[width=\linewidth]{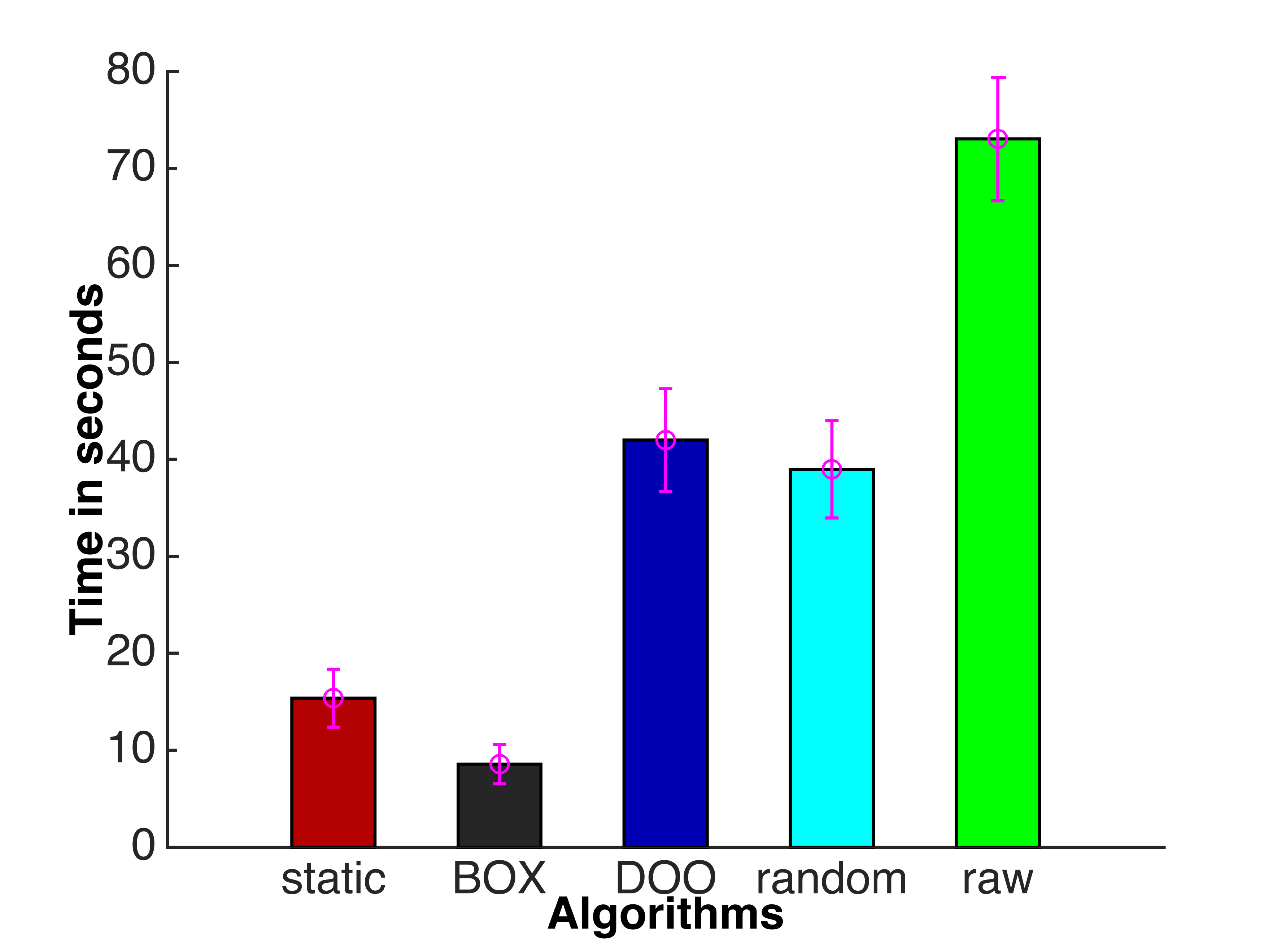}
		\caption{Grasp-and-base selection (RRT)}
		\label{fig:ff_pick_base}
	\end{subfigure} 
	\begin{subfigure}[b]{0.31\textwidth}	
		\centering
		\includegraphics[width=\linewidth]{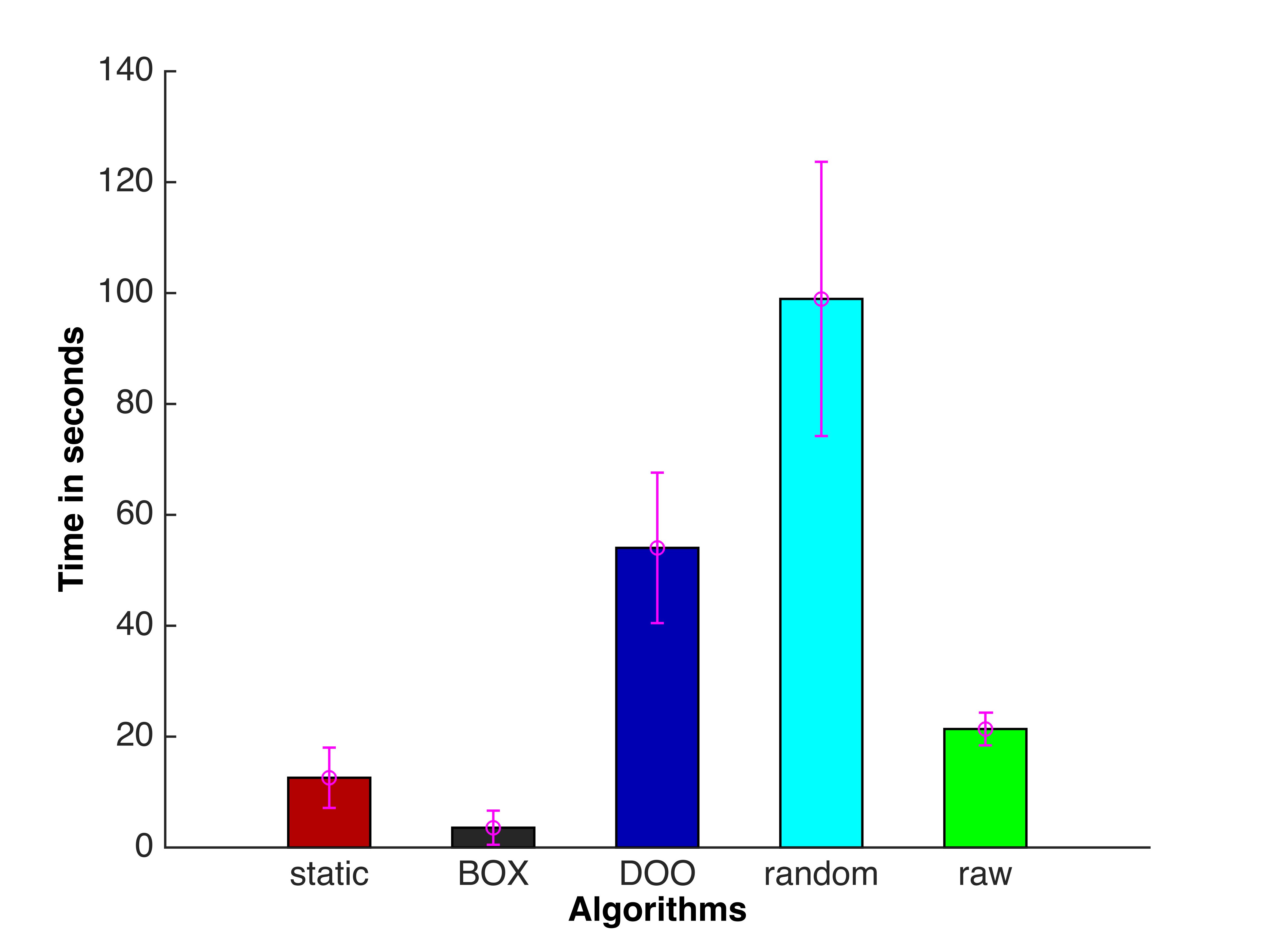}
		\caption{Pick-and-place (RRT)}
		\label{fig:ff_biggest}
	\end{subfigure} 
	\begin{subfigure}[b]{0.31\textwidth}
		\includegraphics[width=\linewidth]{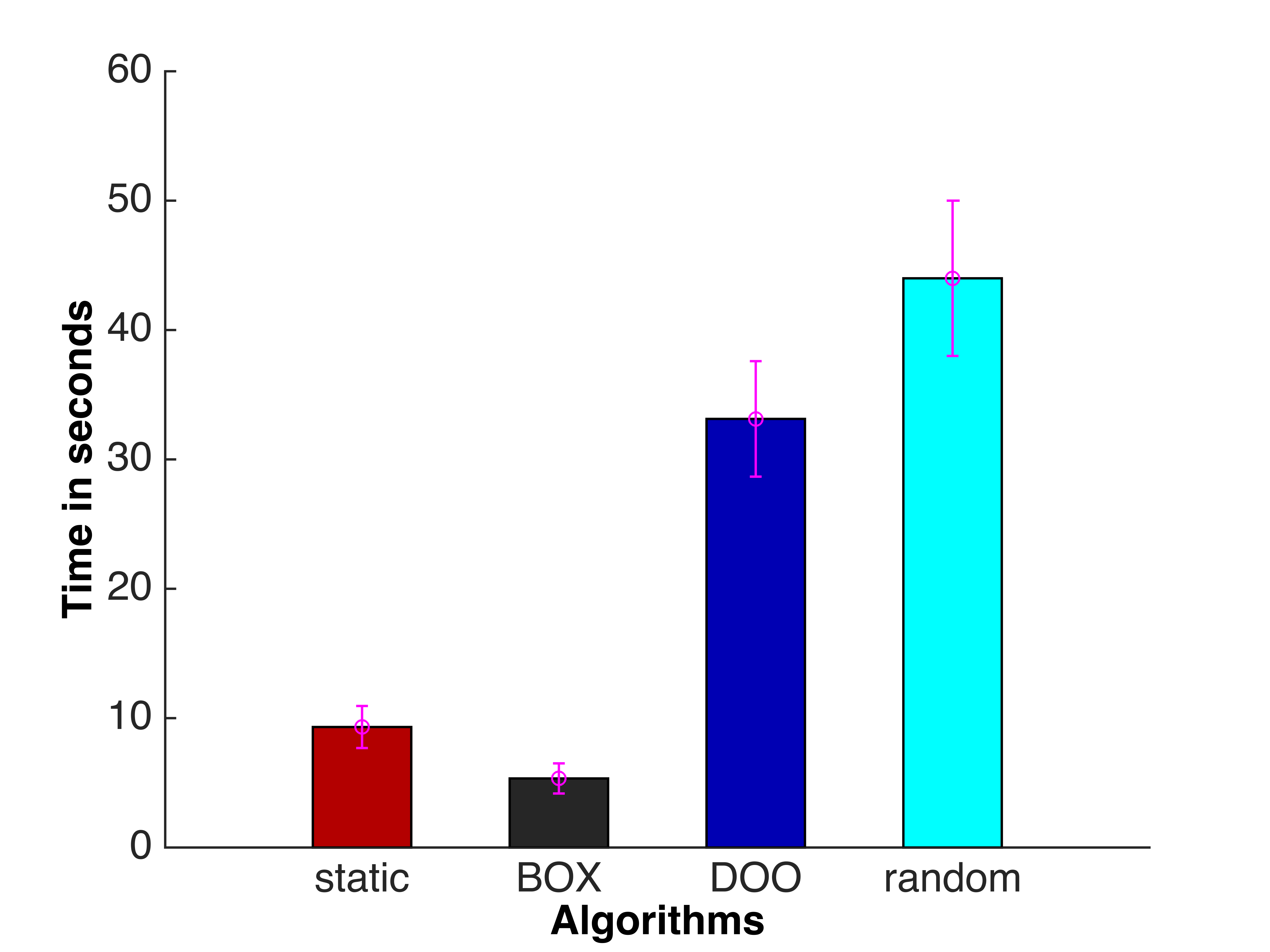}
	\caption{Grasp-selection (TrajOpt)}
	\label{fig:ff_arm_and_grasp_TrajOpt}
	\end{subfigure}
	\begin{subfigure}[b]{0.31\textwidth}
		\includegraphics[width=\linewidth]{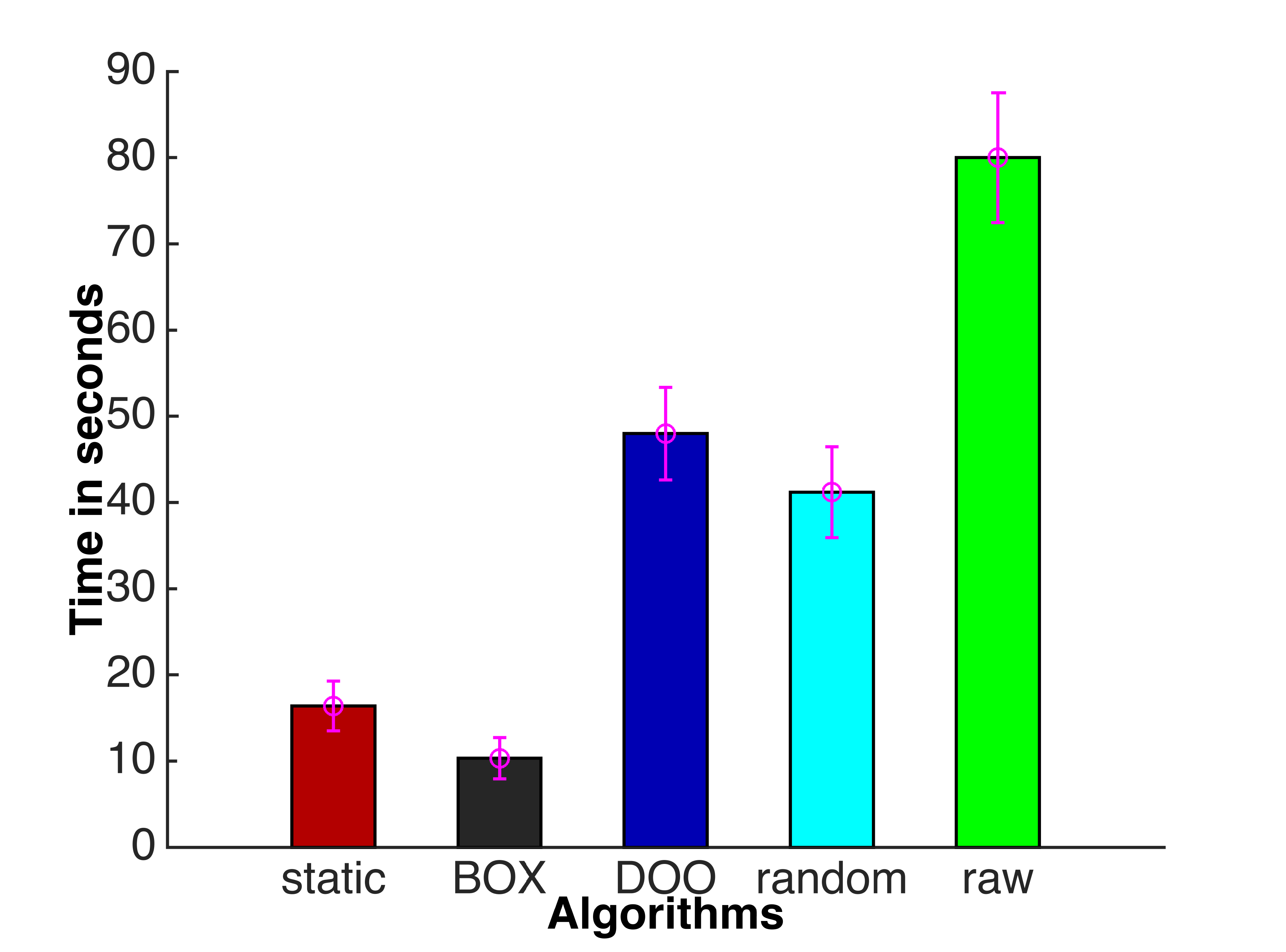}
	\caption{Grasp-and-base selection (TrajOpt)}
	\label{fig:ff_pick_base_TrajOpt}
	\end{subfigure} 
	\caption{LOOCV estimate of time to find first feasible solution, for
	  each method in different domains. Whiskers indicate 95\% confidence interval on mean.
The top row uses RRT and the bottom row uses TrajOpt}
\end{figure*} 
\begin{figure*}

\centering
	\begin{subfigure}[b]{0.31\textwidth}
	\includegraphics[width=\linewidth]{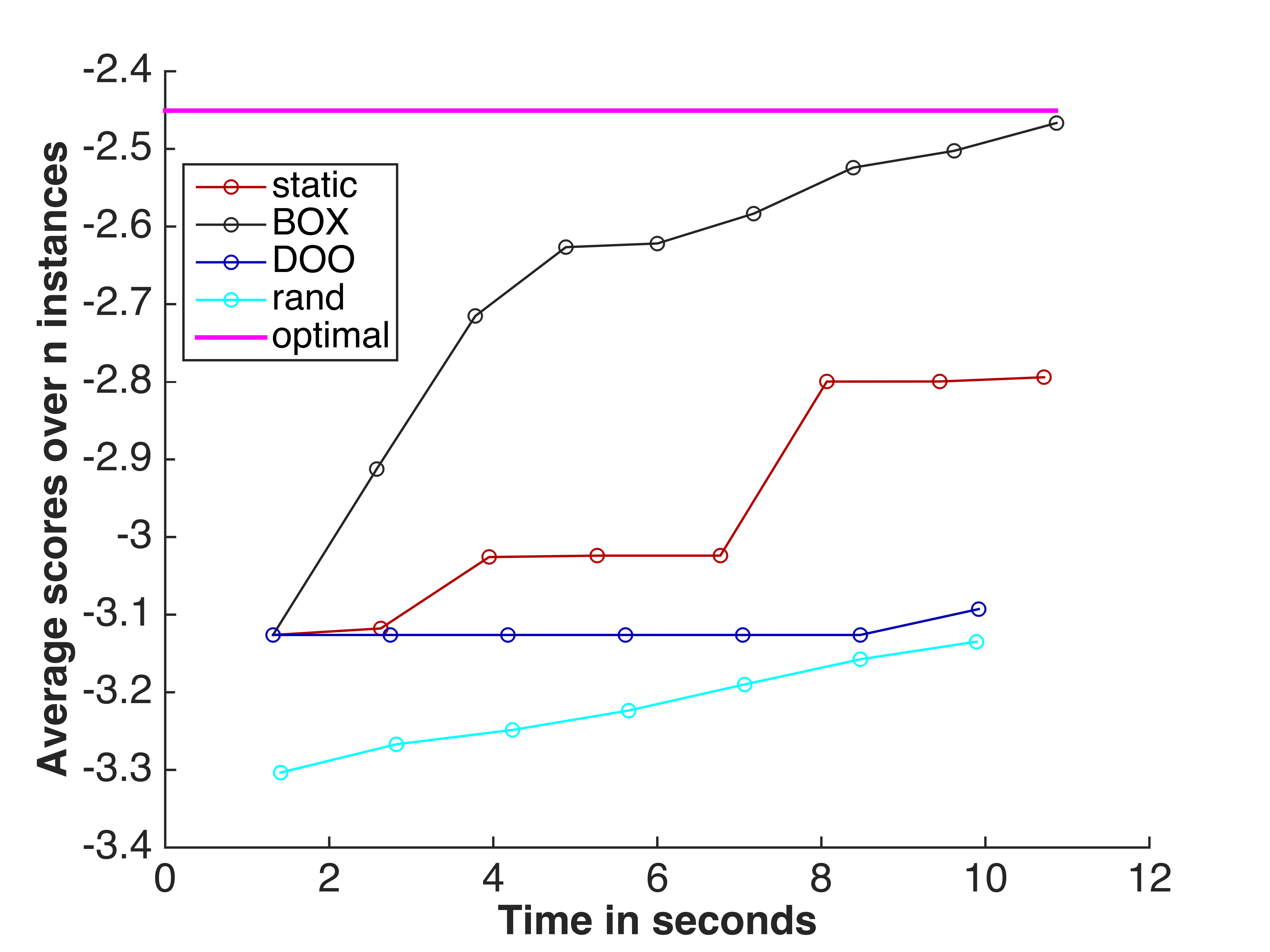}
	\caption{Grasp-selection (RRT)}
	\label{fig:arm_and_grasp_t}
	\end{subfigure}
	\begin{subfigure}[b]{0.31\textwidth}
	\includegraphics[width=\linewidth]{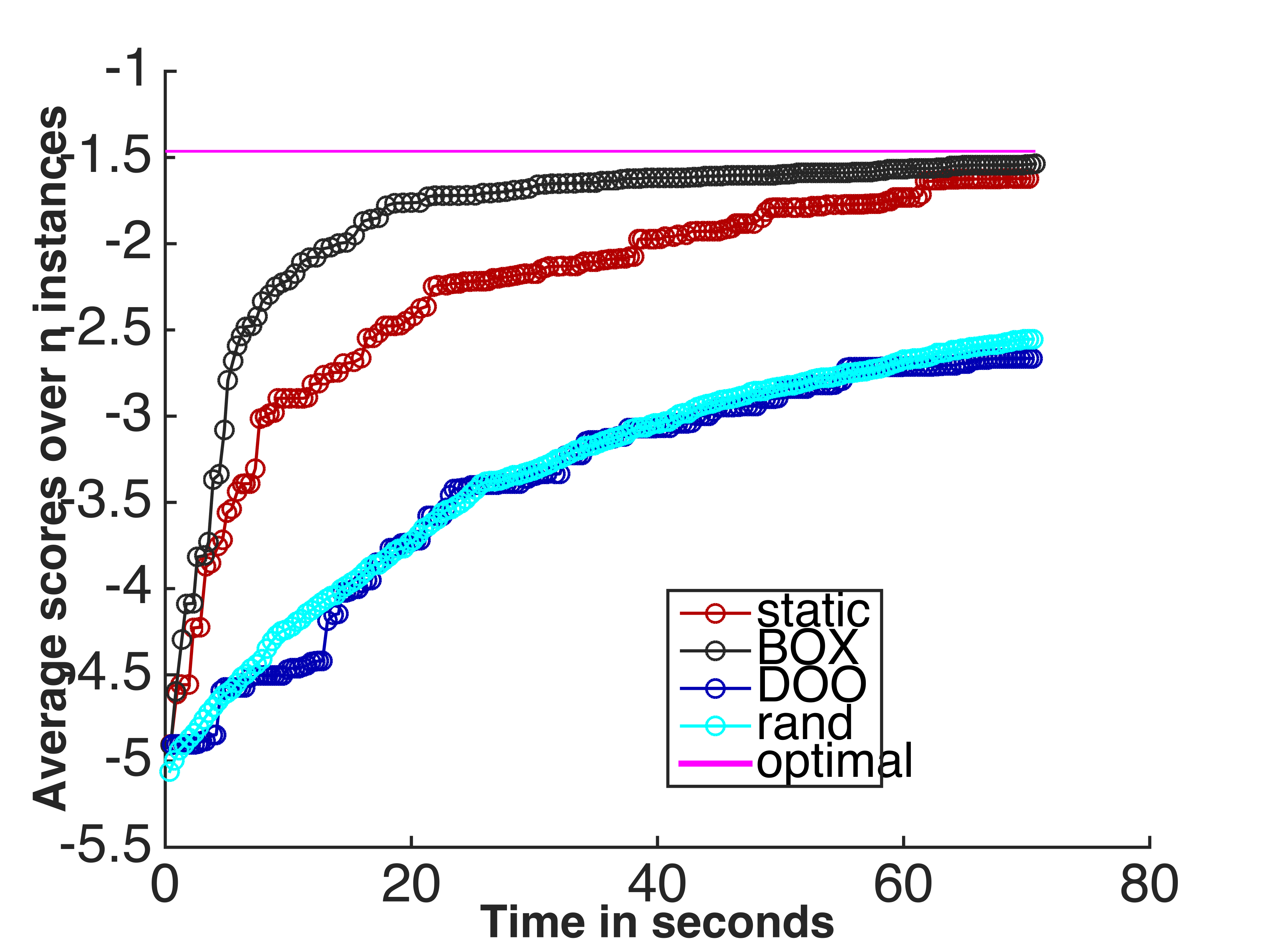}
	\caption{Grasp-and-base selection (RRT)}
	\label{fig:pick_base_t}
	\end{subfigure} 
	\begin{subfigure}[b]{0.31\textwidth}
	\includegraphics[width=\linewidth]{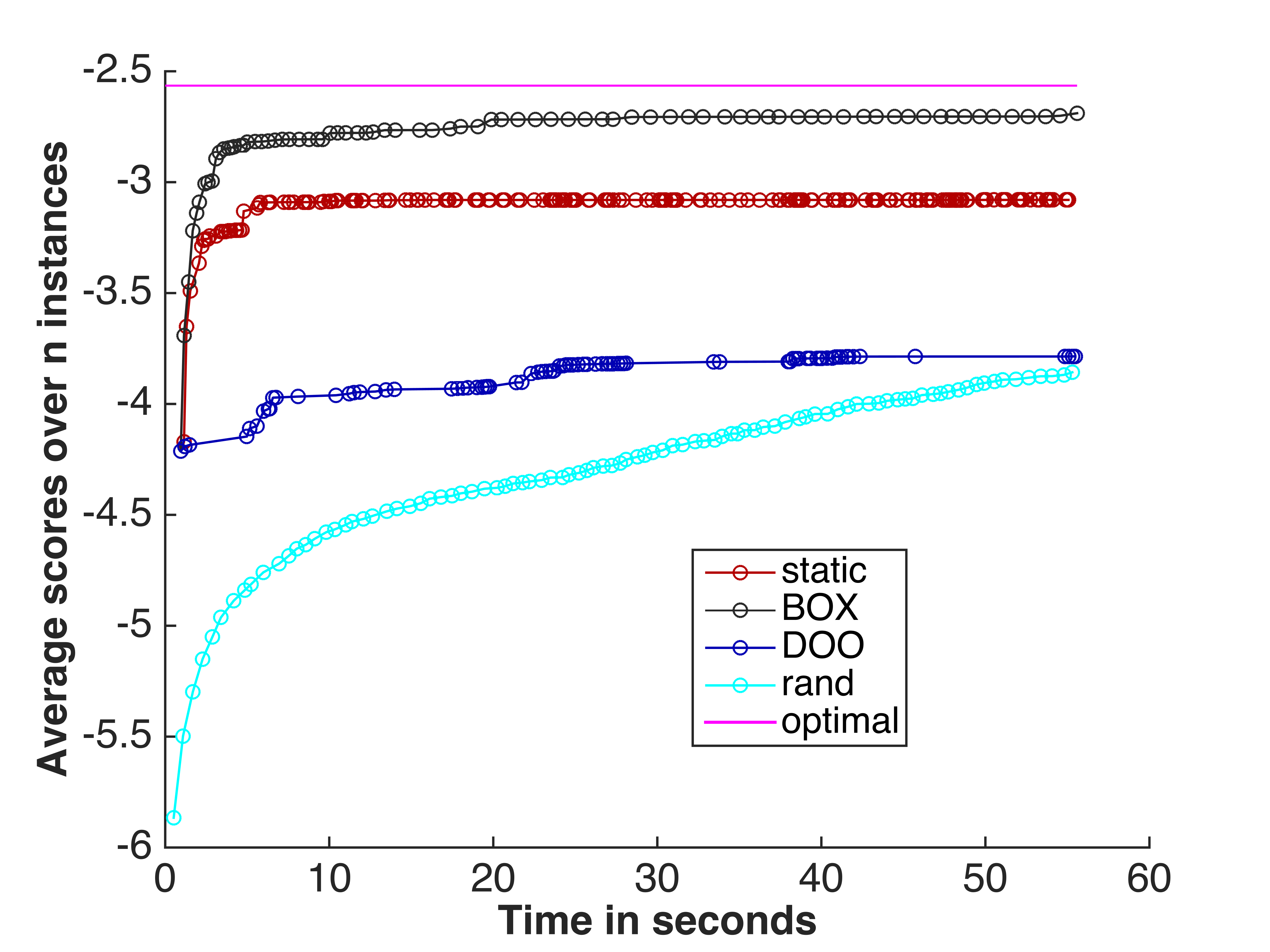}
	\caption{Pick-and-place (RRT)}
	\label{fig:biggest_t}
	\end{subfigure} 
	\begin{subfigure}[b]{0.31\textwidth}	
		\centering
		\includegraphics[width=\linewidth]{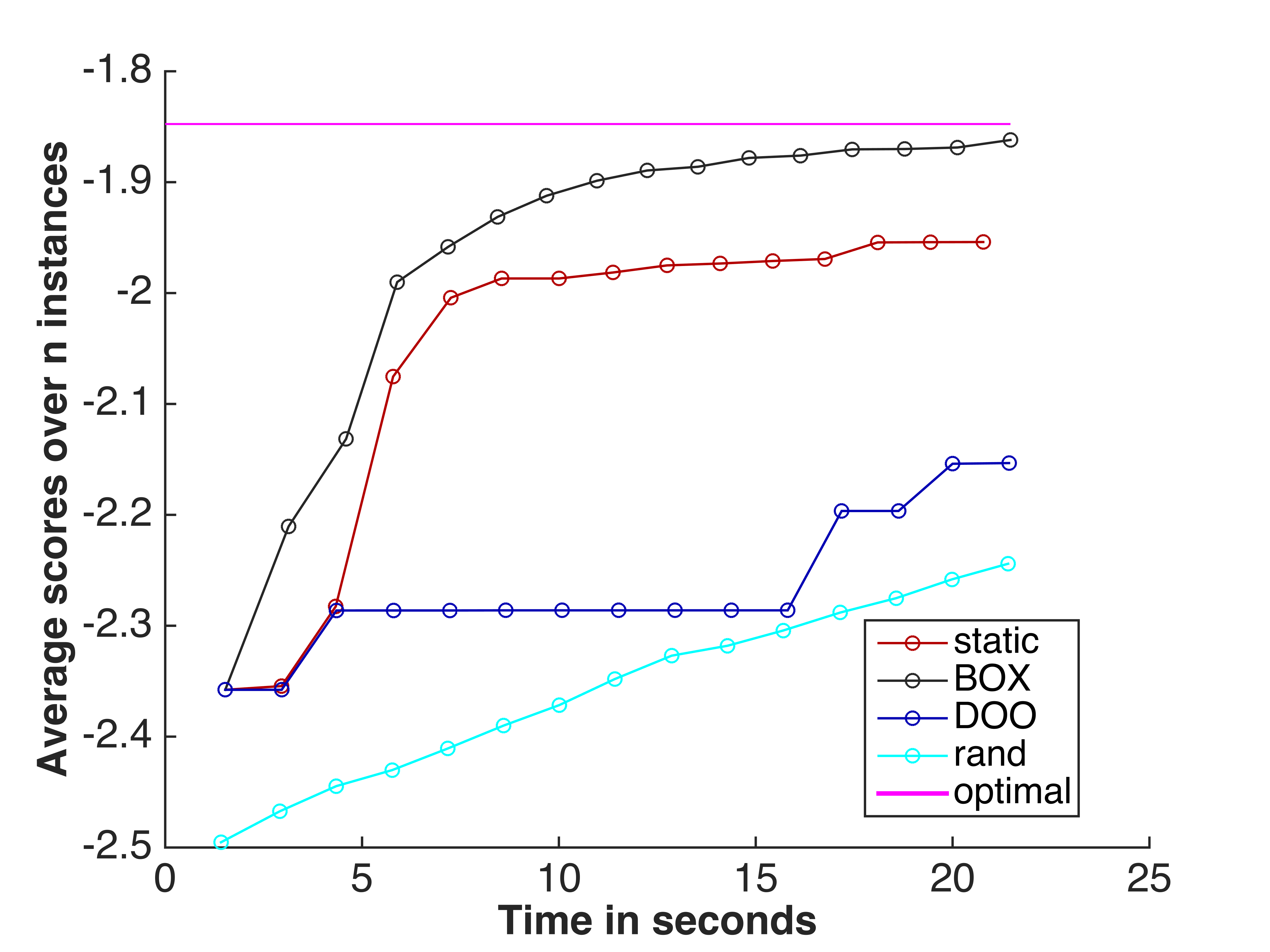}
		\caption{Grasp-selection (TrajOpt)}
		\label{fig:arm_and_grasp_t_TrajOpt}
	\end{subfigure} 
	\begin{subfigure}[b]{0.31\textwidth}	
		\centering
		\includegraphics[width=\linewidth]{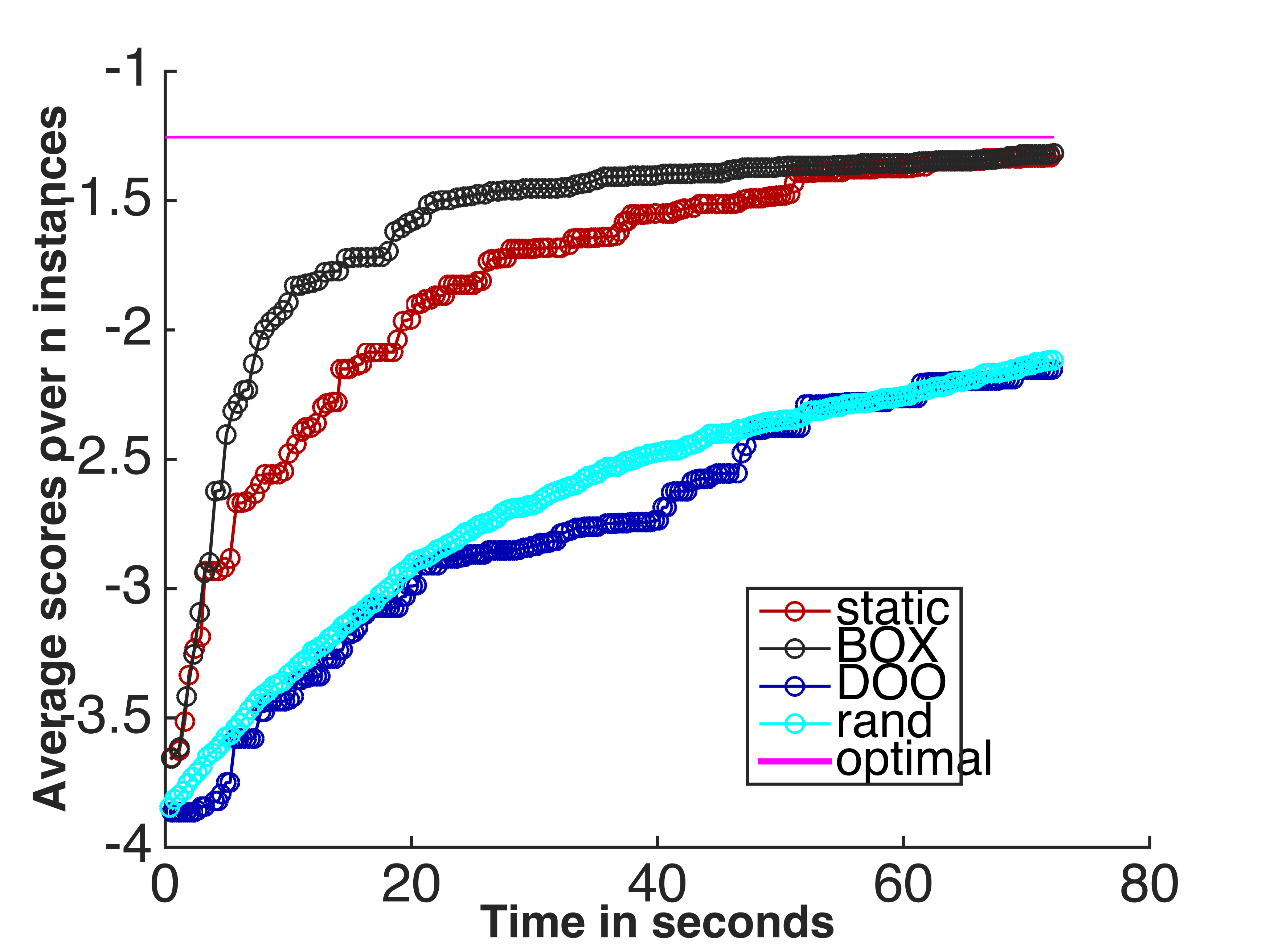}
		\caption{Grasp-and-base (TrajOpt)}
		\label{fig:pick_base_t_TrajOpt}
	\end{subfigure} 
	\caption{Solution score versus run time for different algorithms in
		  various domains. The time axis goes until the first algorithm reaches
		95\% of the optimal score, marked with magenta. This optimal line is obtained by taking the
$\theta$ from $\apprxsolconspace$ that achieved maximum score for each problem instance. 
		The top row uses RRT and the bottom row uses TrajOpt.}
\end{figure*}

\subsection{6.2 Grasp-and-base selection domain} 
In this experiment, we evaluate how the score-space algorithms perform when
we construct the matrix $\apprxsolconspace$ by sampling from a continuous space.
Here, the robot needs to search for a base configuration, a left arm
pre-grasp configuration, and a feasible path between these configurations
to pick an object.

A planning problem instance is again defined by the arrangement of
objects.  Figure \ref{fig:gb_domain} shows three different training
problem instances.  We have 20 rectangular boxes as obstacles, all
resting on the two tables both of which remain fixed in all instances.
For each of the red obstacles and the blue target object, the $(x,y)$
location and orientation in the plane of the table are
randomly chosen subject to the constraint that they are not in
collision.  It is possible that the problem instances will be
infeasible (the target object is too occluded or kinematically
unreachable by the robot).  The robot always starts at the same
initial configuration.

The robot's active DOFs include its base configuration, torso height, 
and left arm configuration, for a total of 11 DOF. 
A solution constraint for this domain consists of the robot base configuration to pick
the target object, $(x,y,\psi)$, where $\psi$ is an orientation of 
the robot, as well as one of 81 grasps from the previous section.  

\begin{figure*}[htb]
\centering
\includegraphics[scale=0.5]{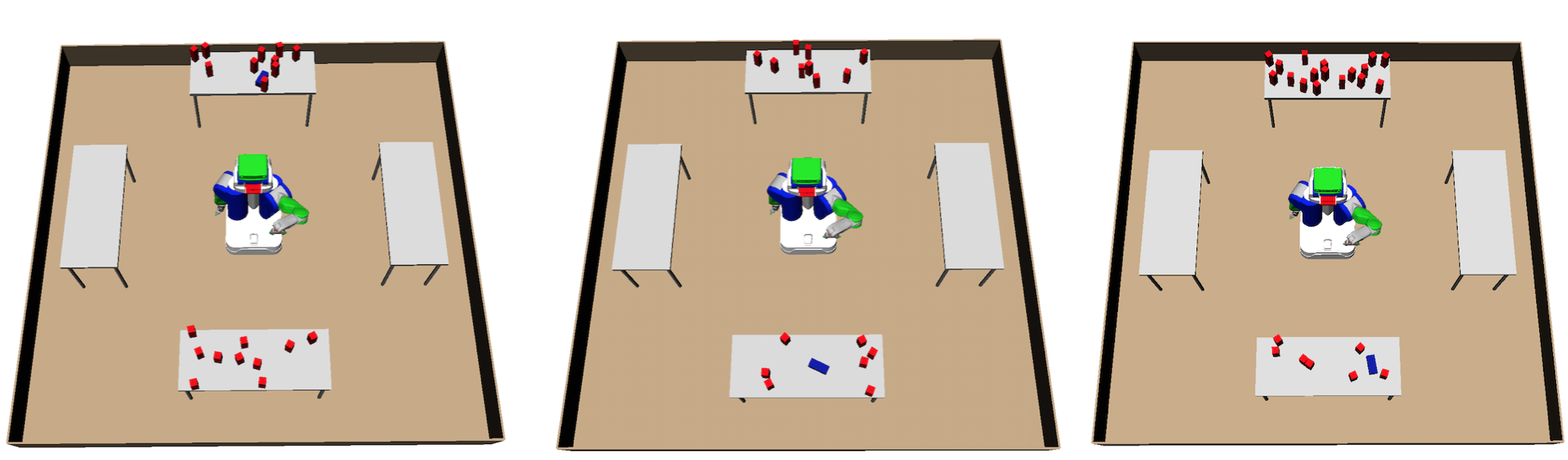}
\caption{Three instances in which the robot
  must select base configuration, grasp, and paths, to pick the target object (blue). 
The poses of the objects are randomly varied between instances.
}
\label{fig:gb_domain}
\end{figure*}

The solution constraints in this case are the grasp $g$,
and the base configuration $k$. 
Given a planning problem instance with no constraints, the
\rawalg~for this
domain performs three sampling procedures, using a uniform random
sampler, backtracking among them as needed to find a feasible solution:
\begin{enumerate}
\item Sample a base configuration, $k=(x,y,\psi)$, from a circular 
region of free configuration space, with radius equal to the
length of the robot's arm, centered at the location of the object.
\item Sample, without replacement, from the 81 grasps until a legal
  one is found, i.e. one for which there is an IK solution in which the robot is
  holding the target object using that grasp in a collision-free
  configuration.
\item Use bidirectional RRT or TrajOpt to find a path for the arm and torso
  between the configurations found in steps 1 and 2.
\end{enumerate}
We assume that the configuration from step 1 is reachable from the
initial configuration. 
To extract a solution constraint from the resulting plan, we simply
return the base configuration from step 1 and the grasp from step 2.

Unlike \rawalg, which has to search for $k$ and $g$,  $\pi(\prob,\solcon)$ 
simply solves the inverse
kinematics and motion planning problems as in the previous example.  The
trajectory of the arm to the pre-grasp configuration, with the base
fixed according to the constraint, is scored according to equation 
\ref{eq:dist_fcn}, with a score of $d$
assigned to problem instances and constraints for which no feasible
solution is found within a fixed number of iterations of the RRT.
The experiments were run on a data set of 1000 problem instances.
The set $\apprxsolconspace$ contained 1000 pairs of grasp and 
robot base configuration, each extracted from a different problem instance.

Figures~\ref{fig:ff_pick_base} and~\ref{fig:ff_pick_base_TrajOpt} show the
time required by each method to find the first feasible plan, using
RRT and TrajOpt as the planner.  The score-space algorithms
perform orders of magnitude better than the other algorithms, with
\boxalg~again outperforming \staticalg. \dooalg~and 
\randalg~do provide some advantage by using previously stored 
solution constraints compared to  \rawalg. \rawalg~has to 
sample in the continuous space of base configurations 
and check whether an IK solution and feasible path exist by 
running IK and path planning. This causes a significant increase
in time to find a solution.

Figure \ref{fig:pick_base_t} compares the solution quality 
vs time when RRT is used and Figure \ref{fig:pick_base_t_TrajOpt}
compares the same quantities when TrajOpt is used. Again, the 
score-space approaches
outperform all other methods, with \boxalg~performing better
than \staticalg, by using the correlation information
from the score space. \dooalg~and \randalg~perform similarly,
mainly because that simple Euclidean distance is not effective
for the hybrid space of base configuration and grasps.

\subsection{6.3 Pick-and-place domain}

In this experiment, with problem instances as shown in
figure~\ref{fig:biggest_domain},
 we introduce solution constraints involving the placements of objects.
Here, the robot needs to pick a large object (shown in
black) up off of a table in one room, carry it through a narrow door,
and place the object on a table. The initial poses of the target 
object and the robot are fixed, but problem instances vary 
in terms of
 the initial poses of 28 obstacles on both the starting and
  final tables, which are chosen uniformly at random on the table-tops
  subject to non-collision constraints, and 
the length of the target object, which is chosen at random from
  three fixed sizes.

\begin{figure*}[htb]
\centering
\includegraphics[scale=0.24]{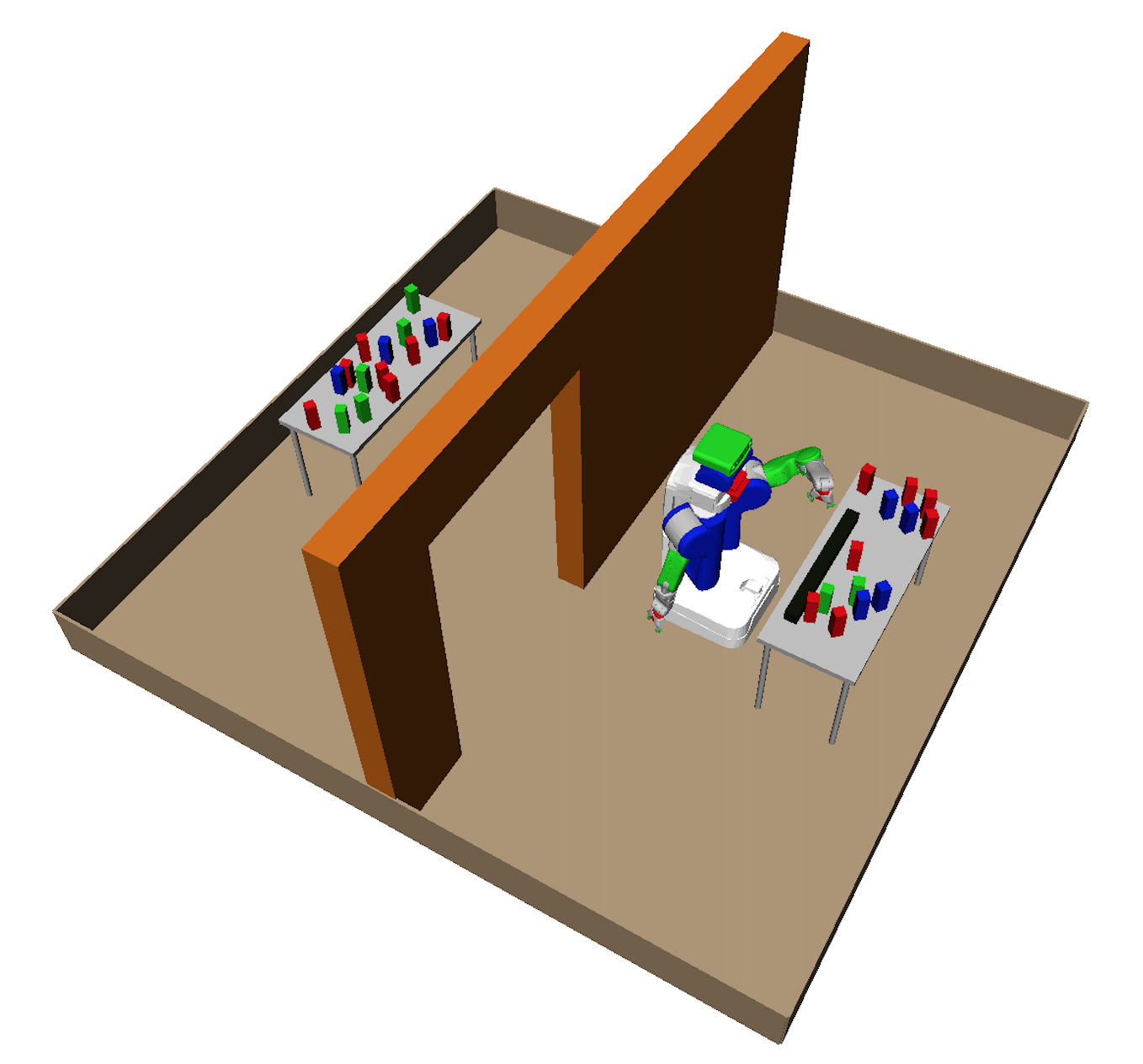}
\includegraphics[scale=0.24]{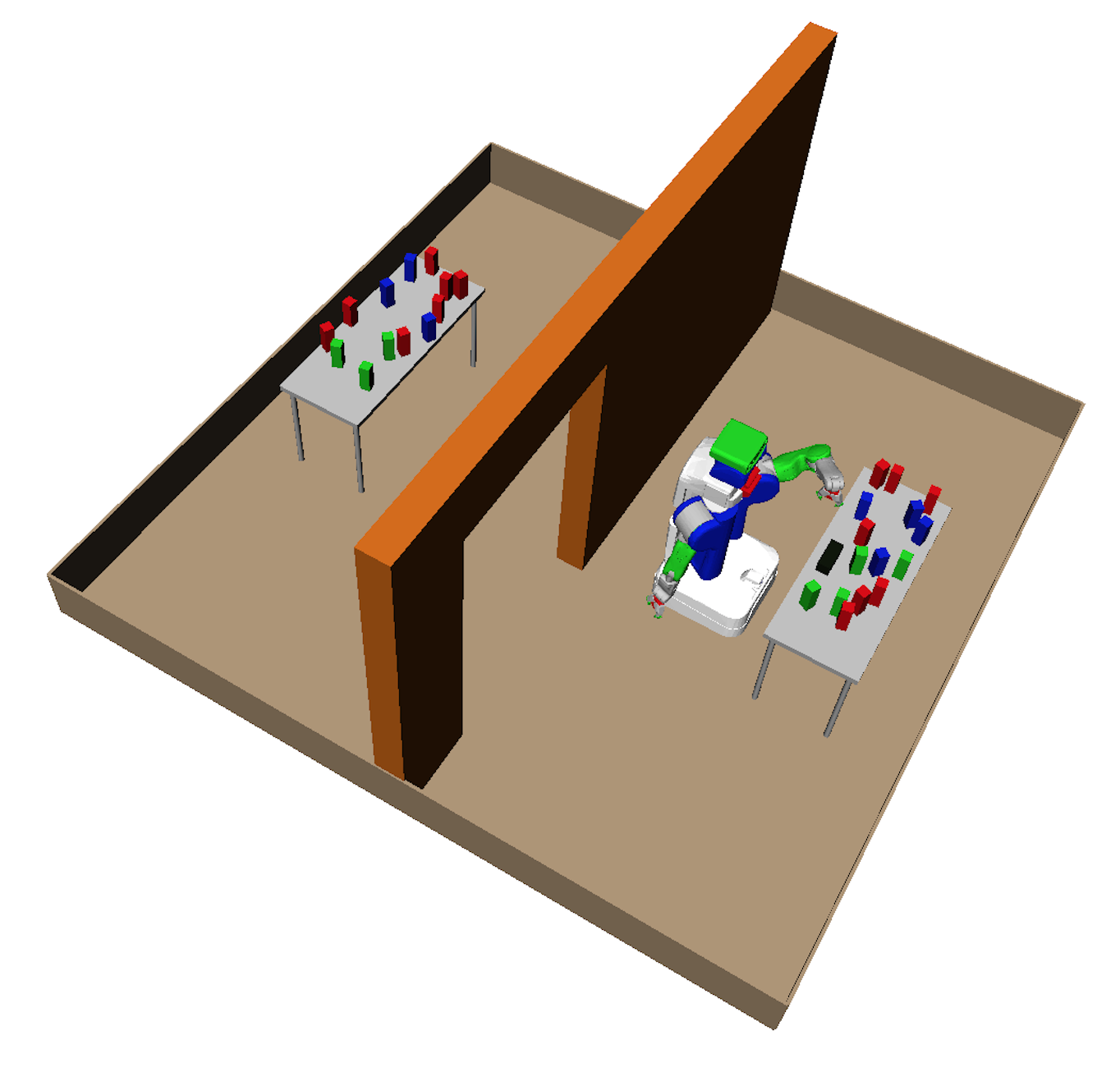}
\includegraphics[scale=0.24]{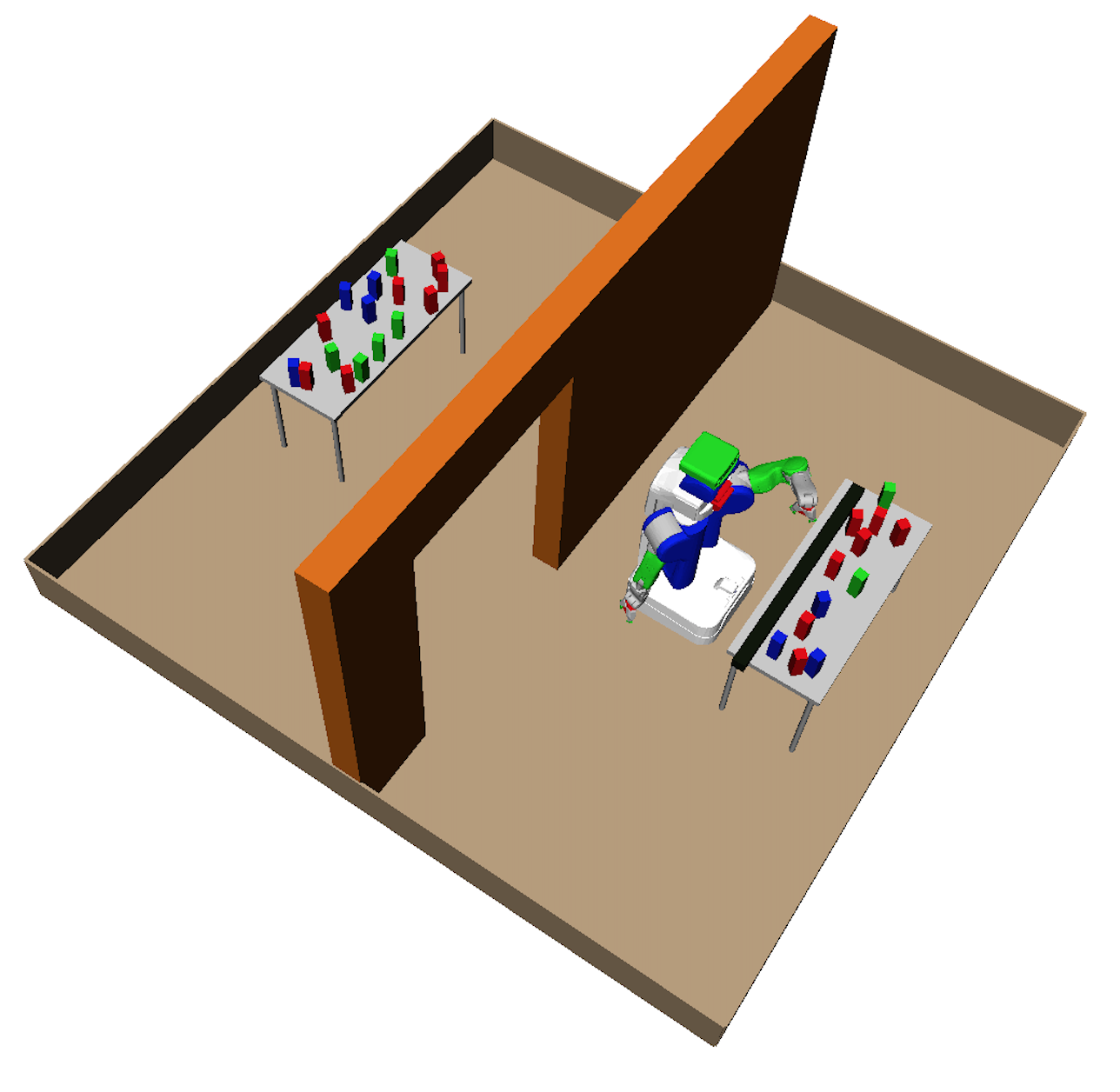}
\caption{Three problem instances from the pick-and-place domain.  The
robot's initial configuration and the black object's initial pose
are fixed across different planning scenes, but other objects' poses
and the black object's length vary.}
\label{fig:biggest_domain}
\end{figure*}

The robot's active DOFs are the same 11 DOFs as in the previous problem domain.
The solution constraints in this domain consists of grasp $g$ to pick
the object, $o$, the placement pose of the object on the table in the
back room, $k_b$, the pre-placement base configuration of the robot
for placing the object at pose $o$, and $k_{sg}$, the subgoal base
configuration for path planning through the narrow passage to $k_b$
from the initial configuration.

Given a problem instance with no constraint, \rawalg~performs 
six sampling procedures, similarly to the previous domain, using a 
uniform random sampler:
\begin{enumerate}
\item Sample a grasp $g$, without replacement, from the 81 grasps
  until a legal one is found.
\item Plan a path for the arm and torso to the pre-grasp configuration found in
step 1. If none is found, choose another grasp.
\item Sample a collision-free object pose $o$ on the table in the other room.
\item Sample $k_b$, the pre-placement base configuration, from a
  circular region of free configuration space around $o$, with radius
  equal to the length of the robot's arms. If none is found, go back to
  the previous step.
\item Plan a path from the initial configuration to $k_b$. If none is found,
go back to the previous step.
\item Plan a path from $k_b$ to a place configuration for putting the
  object down at $o$. If none is found, go back to the previous step.
\end{enumerate}
In contrast, given a solution constraint, $\pi(\prob,\solcon)$ simply
solves for inverse kinematics and path plans.


The experiments were run on a data set of $1500$ problem instances,
with 500 instances per rod size.  
The set $\hat{\Theta}$ contained 1000 tuples of solution-constraint values,
obtained first by running Algorithm \ref{alg:GenTrainData} 
and then randomly subsampling them to reduce the size to 1000.

Figure \ref{fig:ff_biggest} shows the time required by each method
to find the first feasible plan. Again, the score-space algorithms
significantly outperform the other algorithms and \boxalg~outperforms
\staticalg. One noticeable difference between this domain 
and the previous two is that an ineffective solution constraint takes
a long time to evaluate, because computing a path plan or IK
solution for an infeasible constraint is computationally expensive.
This is evident in performance of \randalg~and
\dooalg~which perform worse than \rawalg~as they tend to choose
solution constraints that are infeasible and expensive to
evaluate.

Figure~\ref{fig:biggest_t} shows the average solution score as a
function of computation time. The graphs show a similar trend as in the
previous experiments,
with score-space algorithms outperfoming the other algorithms,
and \boxalg~performing better than \staticalg.
The fact that this domain requires a significant amount of time to try
an ineffective solution constraint is again evident
in \dooalg's plot, where consecutive dots have a large gap between them.
\boxalg~and \staticalg~are able to avoid this problem by exploiting the 
score-space information.

\begin{figure*}[htb]
\centering
\includegraphics[scale=0.35]{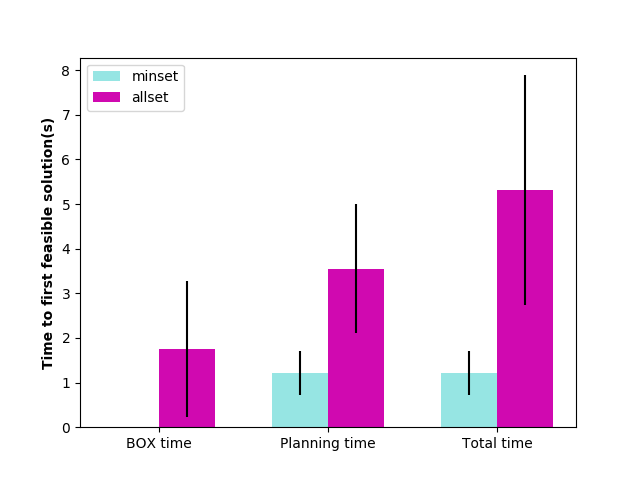}
\includegraphics[scale=0.35]{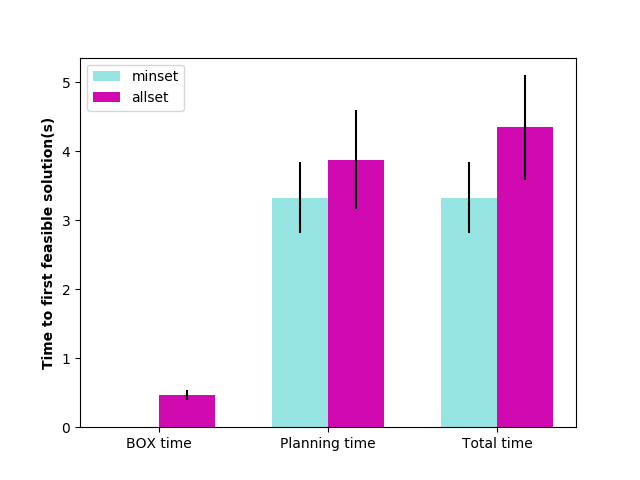}
\includegraphics[scale=0.35]{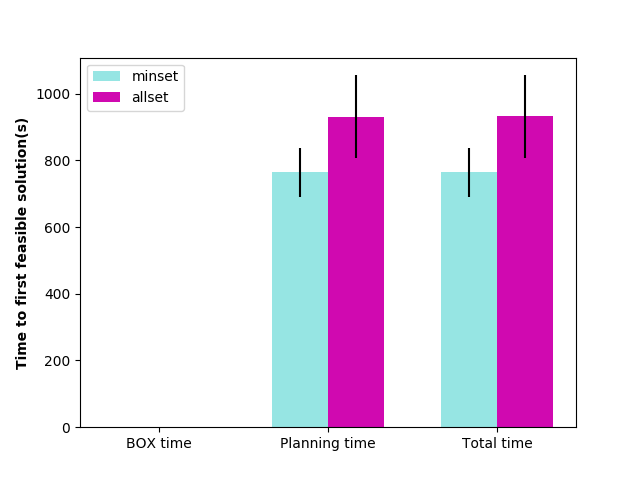}
\caption{A comparison of time to find a first feasible solution for 
grasp-and-base-selection domain (left), pick-and-place domain (middle), and conveyor 
belt domain (right).
\emph{minset} refers to running \boxalg~with the constraint set found by Algorithm~\ref{algo:construct_minset}, 
and \emph{allset} refers to running BOX using the original set $\apprxsolconspace$. \boxalg~time refers to time 
spent (mostly) inverting the covariance matrix, and planning time
refers to time spent on planning using chosen constraints.}
\label{fig:minset_ff}
\end{figure*}

\subsection{6.4 Conveyor belt unloading domain}
In this domain, the robot has
to manipulate box-shaped objects using two-handed grasps. 
The robot's objective is to receive five box-shaped objects 
with various sizes from a conveyor belt and
pack them into a room with a narrow entrance. A problem instance is defined by the
shapes and order of the objects that arrive on the conveyor belt. Examples of
problem instances are shown in Figure~\ref{fig:conv_belt}, and a solved
problem instance is shown in Figure~\ref{fig:solution_convbelt}.

The robot must make a plan for handling all the boxes, including a
grasp for each box, a placement for it in the room, and all of the
robot trajectories. The initial base configuration is always
fixed at the conveyor belt.  After it decides the object placement,
which uniquely determines the robot base configuration, a call to an
RRT is made to find, if possible,  a collision-free path from its
fixed initial configuration at the conveyor belt to the selected
placement configuration. The robot cannot move an object once it has
placed it.

The three previous problems involve an infinite branching factor,
but a relatively shorter planning horizon than this problem: if
we assume that a call to a motion planner is a ``step'' in our plan,
since we are using it as a primitive planner, the grasp-selection domain
has a horizon of one, the grasp-base-selection domain is a 
has a horizon of two (for base planning and arm planning 
for picking an object), and the pick-and-place domain has a horizon of
three (for picking an object, moving robot's base, and then to
place the held object). The conveyor-belt domain requires a horizon of ten, for
picking and placing five objects in total.

Further, this domain is particularly challenging compared to
the previous experiments for two reasons. First, the
robot is operating in an environment with tight free-space,
in which there may or may not be a collision-free path
from one robot configuration to another. If the object placements
are not carefully chosen, calls to the motion planner will be
extremely expensive, either because they are infeasible and will have
to run until a time-out is reached, or because the tolerances are
tight and so even if the problem is feasible, it may run for a long
time or time out. 
Second, they contain a large volume of 
``dead-end'' states that require the task-level planner to backtrack.
For example, if the planner greedily places early objects near the
door, then it will eventually find that it is infeasible to place the
rest of the objects and will have to backtrack to find different
placements of those objects.

As mentioned, we have a significantly longer horizon planning problem
than the previous domains. Therefore, we use graph-search with the 
sampled operators as our~\rawalg. It proceeds as follows:

\begin{enumerate}
\item Place the root node on the search agenda.
\item Pop the node from the agenda with the lowest heuristic value
 (estimated cost to reach the goal)
\item Expand the popped node by generating three operator instances by sampling
their parameters, and add their successor states on the search agenda.
\item If the popped node is a root node, add it back to the queue after expansion
\item If at the current node we cannot sample any feasible operators, then we discard
the node and continue with the next node in the agenda.
\item Go to step 2 and repeat until we arrive at a goal state.  
\end{enumerate}

We have two operators: {\it pick} and {\it place}. To sample parameters for the {\it pick}
 operation, the raw planner~\rawalg~executes the following steps:
\begin{enumerate}
\item Sample a collision-free base configuration, $(x,y,\psi)$, uniformly from a
  circular  region of free configuration space, with radius equal to
  the length of the robot's arm, centered at the location of the
  object, using a uniform sampler.
\item With the base configuration fixed at $(x,y,\psi)$, sample
  $(d,h,\gamma)$, where $d$ and $h$ are in the range $[0.5,1]$, and
  $\gamma$ is in the range $[\frac{\pi}{4},\pi]$, uniformly. If an
  inverse kinematics (IK) solution exists for both arms for this
  grasp, proceed to step 3, otherwise restart.
\item A linear path from the current arm configuration to
  the IK solution found in step 2 is planned.
\end{enumerate}
At each stage, if a collision is detected, this means that the
sampled parameters are infeasible, so sampling proceeds from
the step 1

We assume that the conveyor belt drops
objects into the same pose, and the robot can always reach them from
its initial configuration near the conveyor belt, so we omit step
3. From a state
in which the robot is holding an object, it can place it at a feasible
location in a particular region. To sample parameters for {\it place}, 
\rawalg~executes following steps:
\begin{enumerate}
\item Sample a collision-free base configuration, $(x,y,\psi)$, uniformly
from a desired region $R$.
\item Use bidirectional RRT from the current robot base configuration
to $(x,y,\psi)$. 
\end{enumerate}

\begin{figure}[htb]
\centering
\includegraphics[scale=0.2]{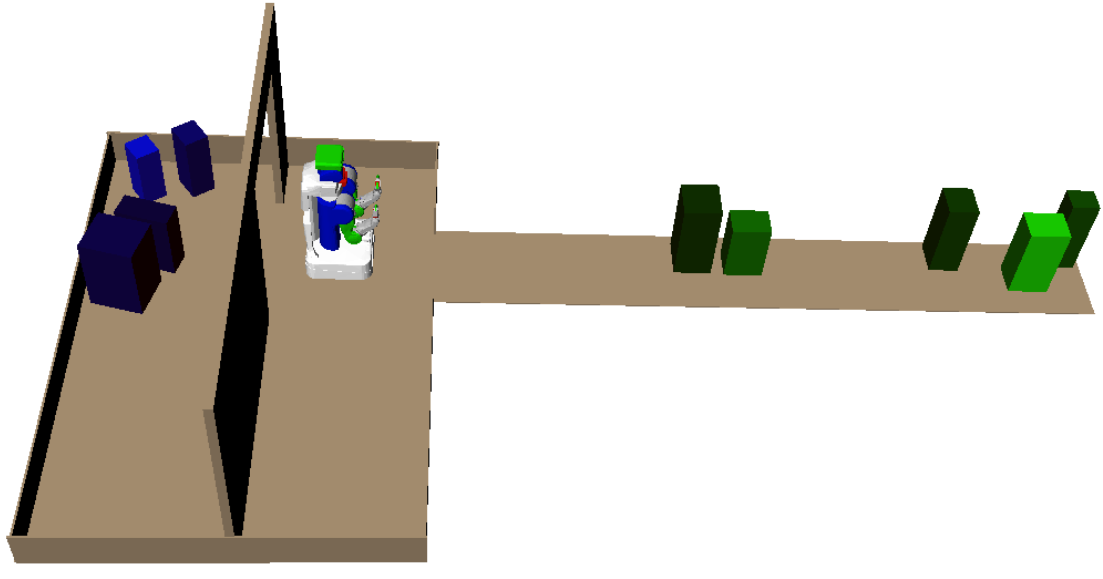}\\
\includegraphics[scale=0.2]{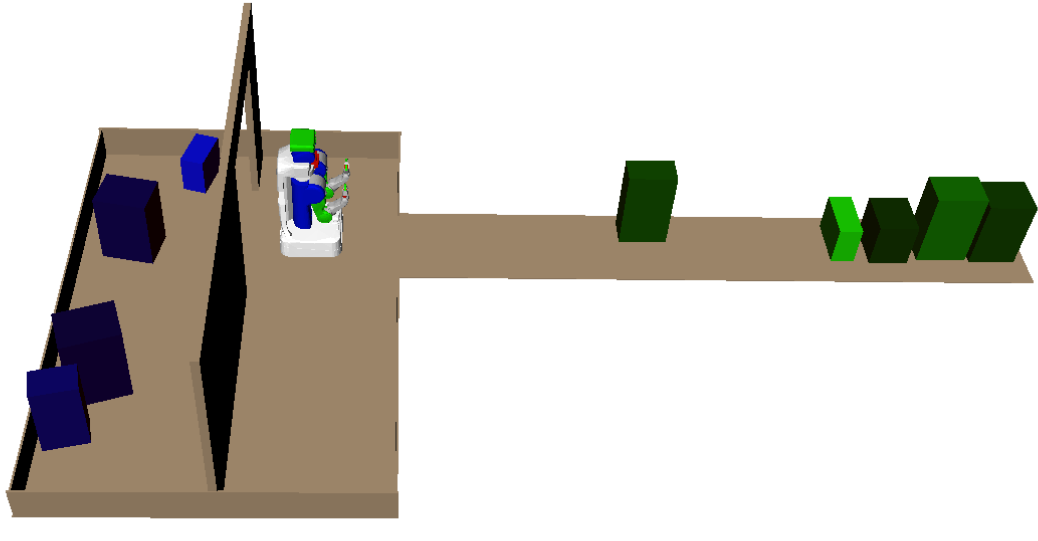}
\caption{Two instances of the conveyor belt domain}
\label{fig:conv_belt}
\end{figure}

\begin{figure}[htb]
\centering
\includegraphics[scale=0.25]{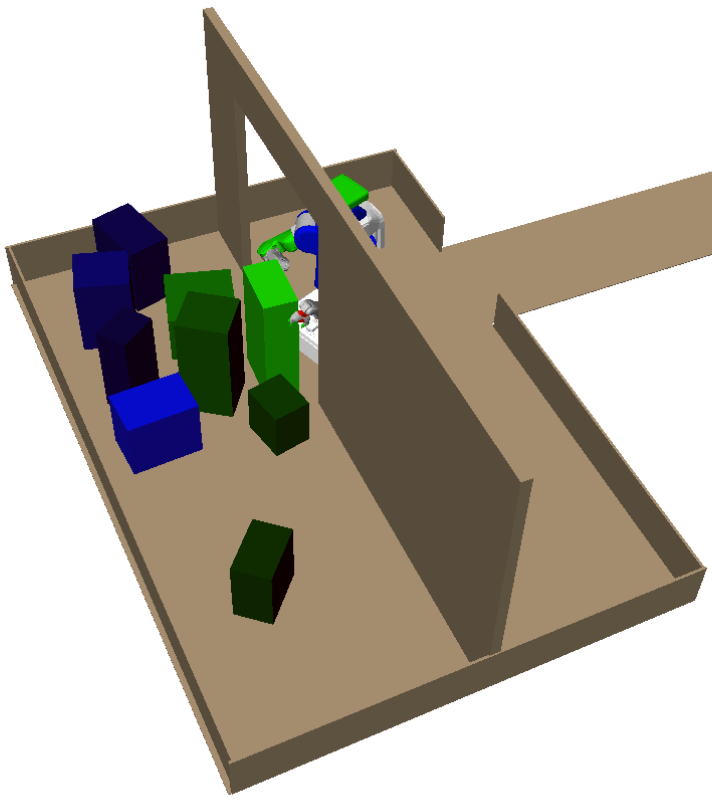}
\caption{A solution for the conveyor belt domain}
\label{fig:solution_convbelt}
\end{figure}

\begin{figure*}
\centering
  \begin{subfigure}[b]{0.45\textwidth}  
    \centering
    \includegraphics[width=\linewidth]{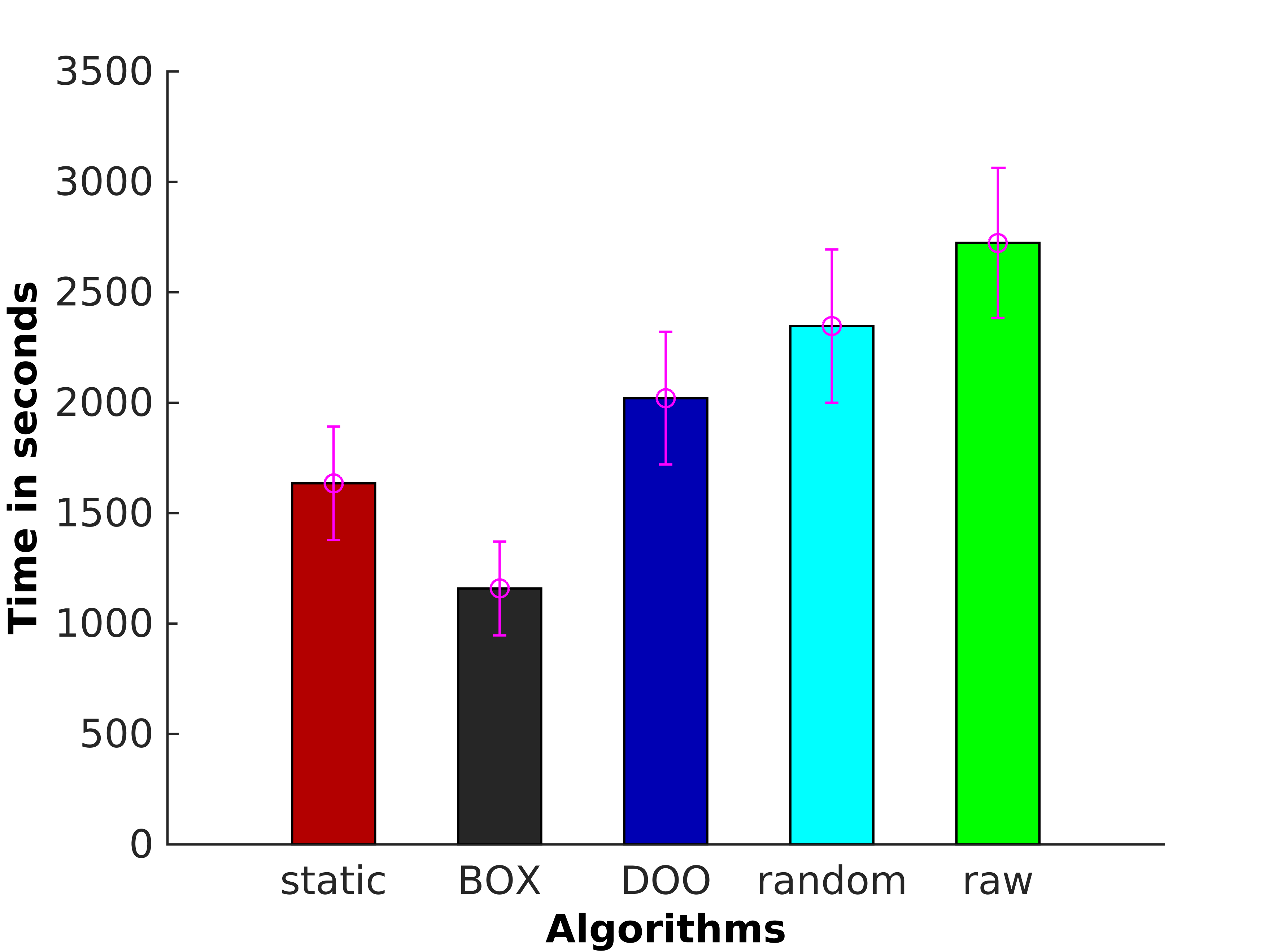}
    \caption{Time to pack four objects for the conveyor belt domain using RRT}
    \label{fig:conv_belt_ff}
  \end{subfigure} 
  \begin{subfigure}[b]{0.45\textwidth}  
    \centering
    \includegraphics[width=\linewidth]{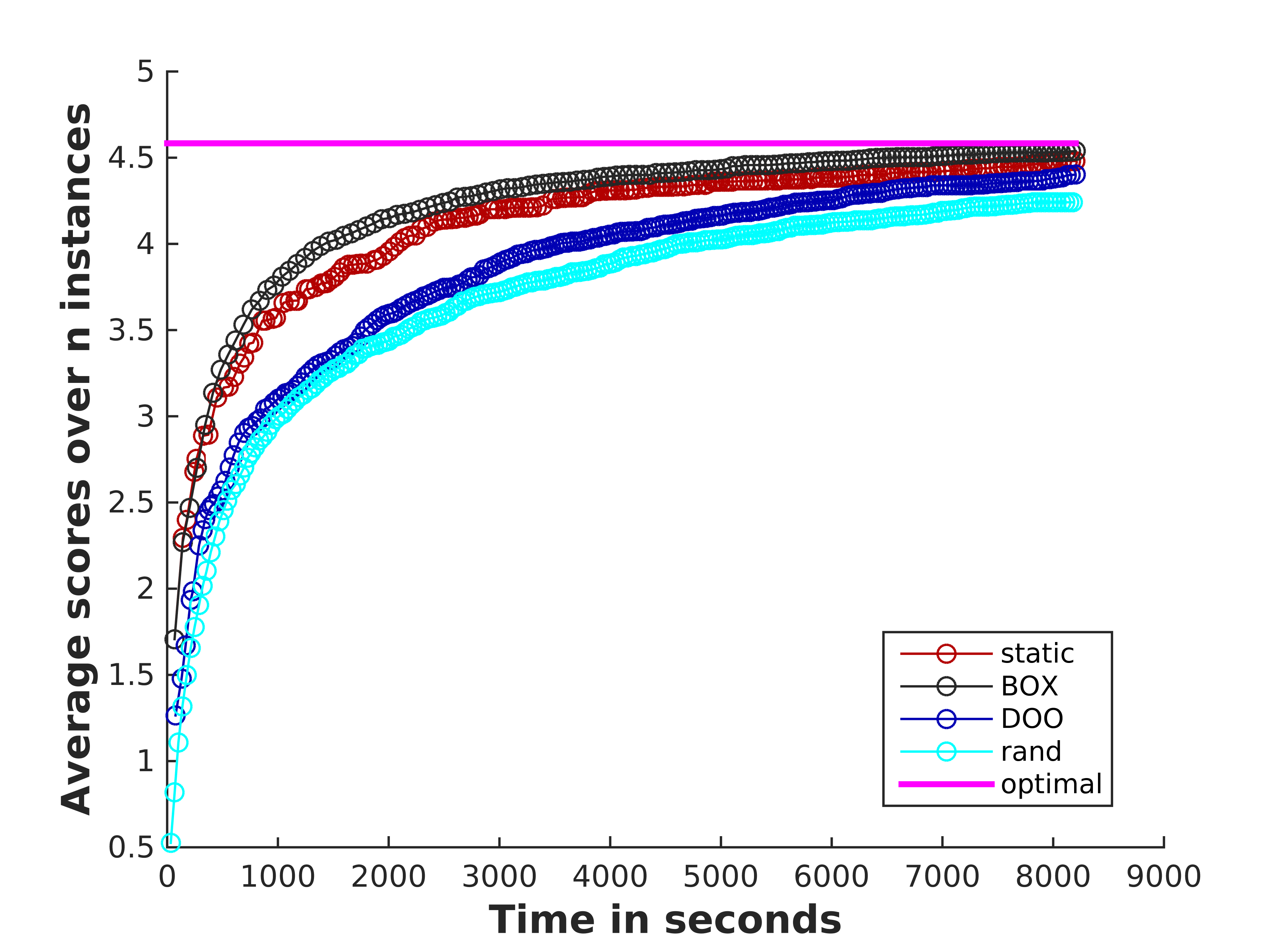}
    \caption{Solution score vs run time for the conveyor belt domain using RRT}
    \label{fig:conv_belt_vs_time}
  \end{subfigure} 
  \caption{LOOCV estimates of different performance metrics for the conveyor belt domain}
\end{figure*} 

The experiments were run on a dataset of 468 problem instances. The solution
constraints in this domain consists of the base configuration for the {\it pick} operator,
and the base configuration for {\it place} operator. The size of $\apprxsolconspace$
was 400, obtained by running Algorithm \ref{alg:GenTrainData} on total of 468 problem instances 
and then randomly subsampling them to reduce the size to 400.

Figure \ref{fig:conv_belt_ff} shows the time required by each method
to pack four objects. Again, the score-space algorithms
significantly outperform the other algorithms and \boxalg~outperforms
\staticalg. In this domain, each evaluation of a constraint is significantly
longer than in the previous domains, because trying a constraint
 involves calls to an RRT up to ten times, some of which may be an infeasible
motion planning problem. We also notice that \dooalg~does not perform as
badly as in the previous domains, because the constraints are defined on
the base poses of robots, in which Euclidean distance is a reasonable 
metric.

Figure~\ref{fig:biggest_t} shows the average solution score as a
function of computation time. The graphs show a similar trend as in the
previous experiments,
with score-space algorithms outperfoming the other algorithms,
and \boxalg~performing better than \staticalg.
Compared to the previous domains, the difference between \boxalg~and \staticalg~is
much smaller. This is because in this domain, any strategy 
that packs objects inside first and then gradually 
towards the narrow entrance will tend to have high scores. Therefore, 
\staticalg, which uses the mean of the scores, works reasonably well, although
\boxalg~generally does better.

\subsection{6.5 Experiments with optimally minimal set}
The purpose of this experiment is to verify our hypothesis that reducing
the size of the constraint set reduces the time for matrix inversion 
involved in updating the mean and covariance matrix in \boxalg, as well as the time to find
the first feasible solution. We test Algorithm~\ref{algo:construct_minset} in
the first three problems previously solved using \boxalg~with full constraints, the
grasp-and-base-selection domain, pick-and-place domain, and conveyor-belt domain,
 and provide comparisons. We omit the grasp-selection domain because it already
has a small constraint set of size 162, compared to 1500,1000, and 500 for
the other three domains.

Using Algorithm~\ref{algo:construct_minset}, we were
able to reduce the constraint set size significantly. 
For the grasp-and-base-selection domain, the algorithm reduced it to 41 from 1500 
for the pick-and-place domain it reduced the constraint
set size to 69 from 1000, and for the conveyor-belt domain it reduced from 500
to 76.

Figure~\ref{fig:minset_ff} shows the times to find a first-feasible 
solution for these domains using the reduced constraint set. As we can see,
it reduces both planning time and \boxalg~time,
 which confirms our hypothesis. In fact, in all of the
domains, it reduced the \boxalg~time, which mostly consists of
covariance matrix inversion time, to almost zero. In terms of reduction
in planning times, the reduction was approximately a factor of around 3
for the grasp-and-base-selection domain. The reduction is smaller in
the pick-and-place and conveyor belt domains, although there is
a notable reduction. The reason that the reduction is smaller
in these two domains is because there 
generally is a smaller reduction in variance of scores of other constraints
 compared to the grasp-and-base-selection domain when one constraint is evaluated.

\section{7. Discussion and future work}
In this paper, we proposed an algorithm for learning to guide 
a planner for task-and-motion planning problems by addressing
three important questions: what to predict, how to represent
a planning problem instance, and how to transfer planning 
knowledge from one instance to another. 

In order to trade-off between the burden on the learning algorithm and
the planner, we proposed to predict constraints on the planning
process rather than a complete solution. To eliminate the bias and
cumbersome feature design, we introduced a score-space 
representation, in which we
construct a representation of a problem instance using a set of
scores of plans that satisfy a pre-built discrete set of constraints
 on-line. To transfer knowledge, we proposed~\boxalg, an algorithm that
tries to both accurately construct a score-space representation of
the given problem instance and choosing a constraint to try next
from the given set.

As an extension to our original work, we also proposed an approach
for reducing the constraint-set size. This is motivated by
the fact that \boxalg~requires inverting a matrix, which gets
larger as the size of the constraint-set increases.
This algorithm effectively reduced the constraint-set size,
 which lead to reduced covariance matrix inversion time 
and reduced number of evaluations of constraints.  
We demonstrated effectiveness of these algorithms
in four challenging TAMP domains. We now discuss
limitations of the current approach and future work.

\subsection{7.1 Fixed plan skeletons}
In this work, we focused on TAMP problems such that even for a
fixed sequence of operators, also known as plan skeletons~\citep{TLPIROS14},
the planner would yield a solution even for different problem instances. For
example for the grasp-and-base selection domain, the sequence
{\it MoveBase,Pick } was sufficient, and for the pick-and-place domain
the sequence {\it Pick,MoveBase,Place} was sufficient for 
different problem instances. As noted by~\cite{TLPIROS14}, the same
plan skeleton can solve a large number of problem instances for
some TAMP problems.

However, there are more general TAMP problems in which 
a fixed plan skeleton would not work. For example, consider 
the problem of making a cup of coffee, where a problem 
instance is defined by the number of spoons of sugar to
put in. For such variation in problem instances, we would need
different plan skeletons depending on the request. 

To deal with this limitation, we are currently 
working on learning high-level constraints that 
constrains the search space of plan skeletons from 
planning experience. Since the number of plan skeletons
is discrete, if we can find a good set of constraints
that reduces the space of plan skeletons to a small but
promising set, then we can construct~$\apprxsolconspace$ for
each skeleton, and then use~\boxalg~appropriately.

\subsection{7.2 Discrete constraints}
To make use of the correlation information among scores
of constraints, our approach builds a 
discrete set of constraints and evaluates them on training
problem instances during the training phase. For some applications,
however, finding a solution that conforms to one of the 
constraints from a selected discrete set might be insufficient to
cope with changes in problem instances. For instance,
consider the task of moving objects to clear a path to the
 target object. Depending on the arrangement of moveable
obstacles, we would need different object placements each time,
and covering all possible such placements in a discrete set
would be difficult. 

For this problem, we are currently looking into
generative models for generating promising constraints
from the original space~$\solconspace$. The idea is
to use the recent advancement in generative model 
learning~\citep{GoodfellowNIPS2014} to generate constraints
with high scores, by training a generative model
for constraints using successful plans. The main
challenge would be how to incorporate score information
appropriately to generative adversarial network.

\bibliographystyle{apalike}
\bibliography{references}

\appendix
\section{Appendix A: Proofs in Section 4}
We first prove Lemma~\ref{lem:rho}.
\begin{proof} (Lemma~\ref{lem:rho}).
Mutual information can be expressed as the difference between entropies,
\begin{align*}
I(f_A;J_A) & = H(J_A) - H(J_A \mid f_A)\\
&= \frac12 \log\det(2\pi e \hat \Sigma_A) - \frac12 \log\det(2\pi e \sigma^2\mI) \\ 
& =  \frac12 \log\det(\sigma^{-2}\hat \Sigma_A).
\end{align*}
For the constraints evaluated by \boxalg, we have
\begin{align*}
I(f_{\Theta_k};J_{\Theta_k}) & = H(J_{\Theta_k}) - H(J_{\Theta_k} \mid f_{\Theta_k})\\
&= \sum_{t=2}^kH(J_{\theta^{(t)}} \mid J_{\Theta_{t-1}}) +  H(J_{\Theta_1})  - \frac12 \log\det(2\pi e \sigma^2\mI)\\
& = \sum_{t=1}^k \frac12\log(2\pi e \hat\Sigma^{(t-1)}_{{\theta^{(t)}}}) - \frac12 k\log(2\pi e \sigma^2)\\
&=\sum_{t=1}^k \frac12\log(\sigma^{-2}\hat\Sigma^{(t-1)}_{{\theta^{(t)}}})\qedhere 
\end{align*}
\end{proof}

Before we continue to the proof of Theorem~\ref{thm:regret}, we introduce some useful lemmas in the following. 
\begin{lem}
\label{lem:con}
Let $\sigma>0$. If $\hS-\sigma^2\mI$ is positive semi-definite, for 
any $t\in[k]$, $\hS^{(t-1)}_{\theta^{(t)}}  -\sigma^2\geq 0$.
\end{lem}
\begin{proof} 
Define $\hS' = \hS-\sigma^2\mI \succeq 0$. By Eq.~\eqref{eqn:update}, for any $t\in[k]$,
\begin{align*}
\hat{\Sigma}^{(t-1)}_{\theta^{(t)}} &= 
\hat{\Sigma}_{\theta^{(t)}}
- \hat{\Sigma}_{\theta^{(t)},\Theta_{t-1}}
(\hat{\Sigma}_{\Theta_{t-1}})^{-1}
\hat{\Sigma}_{\Theta_{t-1},\theta^{(t)}}\\
& =
\hat{\Sigma}'_{\theta^{(t)}} + \sigma^2
- \hat{\Sigma}_{\theta^{(t)},\Theta_{t-1}}
(\hat{\Sigma}'_{\Theta_{t-1}} + \sigma^2\mI)^{-1}
\hat{\Sigma}_{\Theta_{t-1},\theta^{(t)}}\\
&\geq \sigma^2.
\end{align*}
The last inequality is because $\hat{\Sigma}'_{\theta^{(t)}} 
- \hat{\Sigma}_{\theta^{(t)},\Theta_{t-1}}
(\hat{\Sigma}'_{\Theta_{t-1}} + \sigma^2\mI)^{-1}
\hat{\Sigma}_{\Theta_{t-1},\theta^{(t)}} \geq 0$ is a Schur complement of a positive semi-definite matrix \[\begin{bmatrix}
\hS'_{\theta^{(t)}} & \hat{\Sigma}_{\theta^{(t)},\Theta_{t-1}} \\
\hat{\Sigma}_{\Theta_{t-1},\theta^{(t)}} &\hat{\Sigma}'_{\Theta_{t-1}} + \sigma^2\mI
\end{bmatrix}. \qedhere\]
\end{proof}

\begin{cor}[Corollary of Bernoulli's inequality]
\label{cor:math} For any $0\leq x\leq c$ and $a> 0$, we have $x \leq \frac{c\log(1+\frac{ax}{c})}{\log(1+a)}$.
\end{cor}
\begin{proof}
By Bernoulli's inequality, $(1+a)^{\frac{x}{c}} \leq 1+\frac{ax}{c}$. Because $\log(1+a) > 0$, by rearranging, we have $x \leq \frac{c\log(1+\frac{ax}{c})}{\log(1+a)}$.
\end{proof}

\begin{cor}[\cite{ZiAISTATS16}]
\label{lem:normal} Let $\delta_0\in(0,1)$. For any Gaussian variable $x\sim \mathcal N(\mu, \sigma^2), x\in \R,$ with probability at least $1-\delta_0$,
$$  |x-\mu |  \leq \zeta_0\sigma,$$
where $\zeta_0 = (2\log(\frac{1}{\delta_0}))^\frac12$.
\end{cor}
\begin{proof}
Let $z=\frac{\mu - x}{\sigma}\sim\mathcal N(0,1) $. We have 
\begin{align*}
\Pr[z>\zeta_0] &= \int_{\zeta_0}^{+\infty} \frac{1}{\sqrt{2\pi}}e^{-z^2/2}\dif z \\
&=\int_{\zeta_0}^{+\infty} \frac{1}{\sqrt{2\pi}}e^{-(z-\zeta_0)^2/2-\zeta_0^2/2-z\zeta_0}\dif z\\
&\leq e^{-\zeta_0^2/2}\int_{\zeta_0}^{+\infty} \frac{1}{\sqrt{2\pi}}e^{-(z-\zeta_0)^2/2}\dif z\\
&=\frac12 e^{-\zeta_0^2/2}.
\end{align*}
Similarly, $\Pr[z < -\zeta_0] \leq \frac12 e^{-\zeta_0^2/2}.$ Hence, by union bound, $\Pr[|z| > \zeta_0] \leq e^{-\zeta_0^2/2};$ and so, $\Pr[ |x-\mu |  \leq \zeta_0\sigma] > 1-\delta_0$.
\end{proof}

Finally, we prove Theorem~\ref{thm:regret}.
\begin{proof} (Theorem~\ref{thm:regret}). 
By Corollary~\ref{lem:normal}, with probability at least $1-\delta$, 
\begin{align*}
r_k &=  J_{\optsolcon} - \max_{t\in[k]} J^{(t)} \\
& \leq J_{\optsolcon} - J^{(\tau)} \\
&\leq J_{\optsolcon}  - \hmu^{(\tau-1)}_{\theta^{(\tau)}} +  \hmu^{(\tau-1)}_{\theta^{(\tau)}} 
- J_{\theta^{(\tau)}} \\
&\leq \hmu^{(\tau-1)}_{\optsolcon} + \zeta\sqrt{\hS^{(\tau-1)}_{\optsolcon}}  - \hmu^{(\tau-1)}_{\theta^{(\tau)}} +\zeta\sqrt{\hS^{(\tau-1)}_{\theta^{(\tau)}}},
\end{align*}
where $\tau = \argmin_{t\in[T]} \hS^{t-1}_{\theta^{(t)}}$ and $\zeta = (2\log(\frac{1}{\delta}))^\frac12$.
Because of how constraints are selected in each iteration of \boxalg, 
$$\hmu^{(\tau-1)}_{\optsolcon} + \zeta\sqrt{\hS^{(\tau-1)}_{\optsolcon}}  \leq \hmu^{(\tau-1)}_{\theta^{(\tau)}} +\zeta\sqrt{\hS^{(\tau-1)}_{\theta^{(\tau)}}}.$$
Hence we have 
\begin{align}
r_k & \leq 2\zeta \sqrt{\hS^{(\tau-1)}_{\theta^{(\tau)}}}. \nonumber
\end{align}
Applying Corollary~\ref{cor:math}, we get
\begin{align*}
\hS^{(\tau-1)}_{\theta^{(\tau)}} & \leq \frac{1}{k}\sum_{t=1}^k \hS^{(t-1)}_{\theta^{(t)}} \\
&\leq \frac{1}{k}\sum_{t=1}^k (\hS^{(t-1)}_{\theta^{(t)}} -\sigma^2) + \sigma^2\\
& \leq \frac{1}{k}\sum_{t=1}^k  \frac{(c-\sigma^2)\log(1+\frac{(c \sigma^{-2} -1)(\hS^{(t-1)}_{\theta^{(t)}} - \sigma^2)}{c-\sigma^2})}{\log(c \sigma^{-2})} + \sigma^2 \\
& =  \frac{c-\sigma^2}{k\log(c \sigma^{-2})}\sum_{t=1}^k \log(\sigma^{-2}\hS^{(t-1)}_{\theta^{(t)}}) + \sigma^2.
\end{align*}
Notice that here Corollary~\ref{cor:math} applies because $0\leq\hS^{(t-1)}_{\theta^{(t)}}  -\sigma^2\leq \hS_{\theta^{(t)}} -\sigma^2 < c-\sigma^2$ by assumption and Lemma~\ref{lem:con}. Moreover, it is clear that $c\sigma^{-2} -1> 0$.

By Lemma~\ref{lem:rho}, $I(f_{\Theta_k};J_{\Theta_k}) = \frac12 \sum_{t=1}^k\log(\sigma^{-2} \hat\Sigma^{(t-1)}_{\theta^{(t)}}) \leq \rho_k$, so 
$$\hS^{(\tau-1)}_{\theta^{(\tau)}}   \leq  \frac{2(c-\sigma^2)\rho_k}{k\log(c \sigma^{-2})}  + \sigma^2,$$
which implies 
\begin{align*}
r_k &  \leq 2\zeta \sqrt{\frac{2(c-\sigma^2)\rho_k}{k\log(c \sigma^{-2})}  + \sigma^2} \\
& \leq 2 \sqrt{2\log(\frac{1}{\delta}) \left(\frac{2(c-\sigma^2)\rho_k}{k\log(c \sigma^{-2})}  + \sigma^2\right)}. \qedhere
\end{align*}
\end{proof}

\end{document}